\pdfoutput=1
\documentclass{article}

\usepackage{times}
\usepackage{graphicx} %
\usepackage{subfigure} 
\usepackage{epstopdf}

\usepackage{natbib}

\usepackage{algorithm}
\usepackage{algorithmic}

\usepackage{hyperref}
\usepackage[accepted]{icml2016}

\usepackage{url}

\usepackage{bm,color}
\usepackage{mathdots}
\usepackage{float}
\usepackage{amsmath, amssymb,amsthm}
\usepackage{algorithm}
\usepackage{algorithmic}
\usepackage{xspace}
\usepackage{array}
\usepackage{tabu}
\usepackage{graphicx} %
\usepackage{subfigure}
\graphicspath{ {./}{./plots/} }
\usepackage{capt-of}

\newcommand{\glmnet}{\textsc{glmnet}\xspace}
\newcommand\tagthis{\addtocounter{equation}{1}\tag{\theequation}}
\newcommand{\calG}{\mathcal{G}}
\newcommand{\Lone}{L_1}
\newcommand{\Ltwo}{L_2}

\DeclareMathOperator{\dom}{dom}         %

\newcommand{\ve}[2]{\left\langle #1 ,  #2 \right\rangle}    %

\newcommand{\R}{\mathbb{R}}                      %
\newcommand{\Exp}{\mathbb{E}}                      %

\newcommand{\xv}{ {\bf x}}
\newcommand{\yv}{ {\bf y}}
\newcommand{\zv}{ {\bf z}}
\newcommand{\uv}{ {\bf u}}
\newcommand{\vv}{ {\bf v}}
\newcommand{\wv}{ {\bf w}}
\newcommand{\sv}{ {\bf s}}
\newcommand{\alphav}{ {\boldsymbol \alpha}}

\newcommand{\thetav}{ {\boldsymbol \theta}}
\newcommand{\av}{ {\bf a}}
\newcommand{\bv}{ {\bf b}}
\newcommand{\ev}{ {\bf e}}

\newcommand{\0}{ {\bf 0}}
\newcommand{\id}{\bm{\iota}} %

\newcommand{\vc}[2]{#1^{(#2)}}                   %

\newcommand{\norm}[1]{\left\lVert{#1}\right\rVert}

\newcommand{\cC}{\mathcal{C}}
\newcommand{\bP}{\mathcal{P}}
\newcommand{\bD}{\mathcal{D}}
\newcommand{\bB}{\mathcal{B}}
\newcommand{\bL}{\mathcal{L}}
\newcommand{\bS}{\mathcal{S}}

\newtheorem*{rep@theorem}{\rep@title}
\newcommand{\newreptheorem}[2]{%
\newenvironment{rep#1}[1]{%
 \def\rep@title{#2 \ref{##1}}%
 \begin{rep@theorem}}%
 {\end{rep@theorem}}}
\newreptheorem{lemma}{Lemma'}
\newreptheorem{definition}{Definition'}
\newreptheorem{proposition}{Proposition'}
\newreptheorem{theorem}{Theorem'}

\theoremstyle{plain}
\newtheorem{theorem}{Theorem}
\newtheorem{lemma}[theorem]{Lemma}

\newtheorem{remark}{Remark}
\newtheorem{corollary}[theorem]{Corollary}
\theoremstyle{definition}
\newtheorem{definition}{Definition}

\usepackage[draft]{todonotes} %
\icmltitlerunning{Primal-Dual Rates and Certificates}
\usepackage{enumitem}
\begin{document}

\twocolumn[
\icmltitle{
Primal-Dual Rates and Certificates
}

\icmlauthor{Celestine D{\"u}nner}{cdu@zurich.ibm.com}\vspace{-1pt}
\icmladdress{IBM Research, Z\"urich, Switzerland}
\vspace{-1mm}
\icmlauthor{Simone Forte}{fortesimone90@gmail.com}\vspace{-1pt}
\icmladdress{ETH Z\"urich, Switzerland}
\vspace{-1mm}
\icmlauthor{Martin Tak\'a\v{c}}{Takac.MT@gmail.com}\vspace{-1pt}
\icmladdress{Lehigh University, USA}
\vspace{-1mm}
\icmlauthor{Martin Jaggi}{jaggim@inf.ethz.ch}\vspace{-1pt}
\icmladdress{ETH Z\"urich, Switzerland}

\icmlkeywords{boring formatting information, machine learning, ICML}

\vskip 0.3in
\vspace{-1mm}
]

\begin{abstract}
We propose an algorithm-independent framework to equip existing optimization methods with primal-dual certificates. Such certificates and corresponding rate of convergence guarantees are important for practitioners to diagnose progress, in particular in machine learning applications.\\
We obtain new primal-dual convergence rates, e.g., for the Lasso as well as many $L_1$, Elastic Net, group Lasso and TV-regularized problems. The theory applies to any norm-regularized generalized linear model. Our approach provides efficiently computable duality gaps which are globally defined, without modifying the original problems in the region of interest.
\end{abstract}

\setlength{\belowdisplayskip}{5pt} \setlength{\belowdisplayshortskip}{4pt}
\setlength{\abovedisplayskip}{5pt} \setlength{\abovedisplayshortskip}{4pt}

\section{Introduction}

The massive growth of available data has moved data analysis and machine learning to center stage in many industrial as well as scientific fields, ranging from web and sensor data to astronomy, health science, and countless other applications. 
With the increasing size of datasets, machine learning methods are limited by the scalability of the underlying optimization algorithms to train these models, which has spurred significant research interest in recent years.

However, practitioners face a significant problem arising with the larger model complexity in large-scale machine learning and in particular deep-learning methods - it is increasingly hard to diagnose if the optimization algorithm used for training works well or not. With the optimization algorithms also becoming more complex (e.g., in a distributed setting), it can often be very hard to pin down if {\it bad performance of a predictive model either comes from slow optimization, or from poor modeling choices.}
In this light, easily verifiable guarantees for the quality of an optimization algorithm are very useful --- note that the optimum solution of the problem is unknown in most cases. For convex optimization problems, a primal-dual gap can serve as such a certificate. If available, the gap also serves as a useful stopping criterion for the optimizer.

So far, the majority of popular optimization algorithms for learning applications comes without a notion of primal-dual gap.
In this paper, we aim to change this for a relevant class of machine learning problems. We propose a primal-dual framework which is algorithm-independent, and allows to equip existing algorithms with additional primal-dual certificates as an add-on.

Our approach is motivated by the recent analysis of SDCA \cite{ShalevShwartz:2013wl}. We extend their setting to the significantly larger class of convex optimization problems of the form
\[
 \min_{\alphav\in\R^n}  \ f(A\alphav) \,+\, g(\alphav)
\]
for a given matrix $A\in\R^{d\times n}$, $f$ being smooth, and $g$ being a general convex function.
This problem class includes the most prominent regression and classification methods as well as generalized linear models. We will formalize the setting in more details in Section \ref{sec:primaldual}, and highlight the associated dual problem, which has the same structure.
An overview over some popular examples that can be formulated in this setting either as primal or dual problem is given in Table \ref{tbl:MLProblems}.

{\bf Contributions.}
The main contributions in this work can be summarized as follows:\vspace{-2mm}
\begin{itemize}[itemsep=0pt]

 \item Our new primal-dual framework is algorithm-independent, that is it allows users to equip existing algorithms with primal-dual certificates and convergence rates.
 
 \item We introduce a new Lipschitzing trick allowing duality gaps (and thus accuracy certificates) which are globally defined, for a large class of convex problems which did not have this property so far. Our approach does not modify the original problems in the region of interest.
In contrast to existing methods adding a small strongly-convex ($L_2$) term as, e.g., in~\cite{ShalevShwartz:2014dy,Zhang:2015vj}, our approach leaves both the algorithms and the optima unaffected.
 
 \item Compared with the well-known duality setting of SDCA \cite{ShalevShwartz:2013wl,ShalevShwartz:2014dy} which is restricted to strongly convex regularizers and finite sum optimization problems,
our framework encompasses a significantly larger class of problems.
We obtain new primal-dual convergence rates, e.g., for the Lasso as well as many $L_1$, Elastic Net, group Lasso and TV-regularized problems. The theory applies to any norm-regularized generalized linear model. 

 \item Existing primal-dual guarantees for the class of ERM problems with Lipschitz loss from \cite{ShalevShwartz:2013wl} (e.g., SVM) are valid only for an ``average'' iteration.
 We show that the same rate of convergence can be achieved, e.g., for accelerated SDCA, but without the necessity of computing an average iterate.

 \item Our primal-dual theory captures a more precise notion of data-dependency compared with existing results (which relied on per-coordinate information only). To be precise, our shown convergence rate for the general algorithms is dependent on the
\emph{spectral norm of the data}, see also \cite{takavc2013mini,takavc2015distributed}. 

\end{itemize}\vspace{-2mm}

\section{Related Work}

{\bf Linearized ADMM solvers.}
For the problem structure of our interest here, one of the most natural algorithms is the splitting method known as the Chambolle-Pock algorithm (also known as Primal-Dual Hybrid Gradient,  Arrow-Hurwicz method, or linearized ADMM) \cite{Chambolle:2010hk}.
While this algorithm can give rise to a duality gap, it is significantly less general compared to our framework. In each iteration, it requires a complete solution of the proximal operators of both $f$ and $g$, which can be computationally expensive. Its convergence rate is sensitive to the used step-size \cite{Goldstein:2015vn}. Our framework is not algorithm-specific, and holds for arbitrary iterate sequences.
More recently, the SPDC method \cite{Zhang:2015vj} was proposed as a coordinate-wise variant of \cite{Chambolle:2010hk}, but only for strongly convex $g$.

{\bf Stochastic coordinate solvers.}
Coordinate descent/ascent methods have become state-of-the-art 
for many machine learning problems \cite{hsieh2008dual,Friedman:2010wm}.
In recent years, theoretical convergence rate guarantees have been developed for the primal-only setting, e.g., by 
\cite{nesterov2012efficiency,
richtarik2014iteration}%
, as well as more recently also for primal-dual guarantees, see, e.g., \cite{LacosteJulien:2013ue,ShalevShwartz:2013wl,ShalevShwartz:2014dy}.
The influential Stochastic\footnote{Here 'stochastic' refers to randomized selection of the active coordinate.} Dual Coordinate Ascent (SDCA) framework \cite{ShalevShwartz:2013wl} was motivated by the $L_2$-regularized SVM, where coordinate descent is very efficient on the dual SVM formulation, with every iteration only requiring access to a single datapoint (i.e. a column of the matrix $A$ in our setup).
In contrast to primal stochastic gradient descent (SGD) methods, the SDCA algorithm family is often preferred as it is free of learning-rate parameters, and has a fast linear (geometric) convergence rate. SDCA and recent extensions \cite{takavc2013mini,ShalevShwartz:2014dy,Qu:2014uw,shalev2015sdca,Zhang:2015vj} require $g$ to be strongly convex.

Under the weaker assumption of weak strong convexity \cite{necoara2015linear}, a linear rate for the primal-dual convergence of SDCA was recently shown by \cite{ma2015linear}.

In the Lasso literature, a similar trend in terms of solvers has been observed recently, but with the roles of primal and dual problems reversed. For those problems, coordinate descent algorithms on the primal formulation have become the state-of-the-art, as in \glmnet~\cite{Friedman:2010wm} and extensions \cite{ShalevShwartz:2011vo,Yuan:2012wi,Yuan:2010ub}. Despite the prototypes of both problem types---SVM for the $\Ltwo$-regularized case and Lasso for the $\Lone$-regularized case---being closely related \cite{Jaggi:2014co}, we are not aware of existing primal-dual guarantees for coordinate descent on unmodified $L_1$-regularized problems.

\begin{table*}[t]
 \caption{Primal-dual convergence rates of general algorithms applied to \eqref{eq:A}, for some machine learning and signal processing problems which are examples of our optimization problem formulations~\eqref{eq:A} and~\eqref{eq:B}. Note that several can be mapped to either~\eqref{eq:A} or~\eqref{eq:B}.
 $\lambda>0$ is a regularization parameter specified by the user. 
 We will discuss the most prominent examples in more detail in Sections \ref{sec:normreg} and \ref{sec:CD}.\vspace{1mm}
 }
\centering 
\setlength\tabcolsep{2pt} %
\begin{tabular}{|l|lc|l|l|c|}
\hline
\multicolumn{2}{|c}{\textbf{Problem}}  & &
\multicolumn{1}{c|}{ $f/f^*$}  & 
\multicolumn{1}{c|}{$g/g^*$} & 
\multicolumn{1}{c|}{\textbf{Convergence}} \\ \hline \hline 
Regularizer\hspace{0.5mm} & $L_1$ & \eqref{eq:A} & $f=\ell(A\alphav)$ & $g = \lambda \|\alphav\|_1$ &  Thm \ref{thm:convGeneral}, Cor \ref{cor:l1CD}
\\ \cline{2-6}
&Elastic Net & \eqref{eq:A} & $f=\ell(A\alphav)$ & $g=\lambda \left(\frac{\eta}{2}\|\alphav\|_2^2+(1\!-\!\eta)\|\alphav\|_1\right) $ & Thm \ref{thm:convFastSCg},\ref{thm:SDCAthm5}, Cor \ref{cor:elasticCD}
\\ \cline{4-6}
&& \eqref{eq:B} & $f^*\!\!=\lambda\!\left(\frac{\eta}{2}\|\wv\|_2^2+(1\!-\!\eta)\|\wv\|_1\right) $ & $g^*=\ell(-A^\top\wv)$&Thm \ref{thm:convFastSCg},\ref{thm:SDCAthm5}, Cor \ref{cor:elasticCD}
\\ \cline{2-6} 
& $L_2$& \eqref{eq:A} & $f=\ell(A\alphav)$ & $g = \frac{\lambda}{2} \|\alphav\|_2^2$ & Thm \ref{thm:convFastSCg},\ref{thm:SDCAthm5}
\\ \cline{4-6}
&& \eqref{eq:B} & $f^*= \frac{\lambda}{2} \|\wv\|_2^2$ & $g^* =\ell(-A^\top\wv)$ & Thm \ref{thm:convFastSCg},\ref{thm:SDCAthm5}
\\ \cline{2-6}
&Fused Lasso & \eqref{eq:A} & $f=\ell(A\alphav)$ & g = $\lambda \|M\alphav\|_1$ & Thm \ref{thm:convGeneral}
\\ \cline{2-6}
&Group Lasso & \eqref{eq:A} & $f=\ell(A\alphav)$ & $g=%
\lambda \sum_{k} \|\alphav_{\calG_k}\|_2 $, $\{\calG_k\} $ part. of $\{1..n\}\!$& Thm \ref{thm:convGeneral}
\\
\hline
\multicolumn{2}{|l}{{SVM {\small (hinge loss)}}}
& \eqref{eq:A} &$f=\frac{\lambda}{2} \|A\alphav\|_2^2$& $g=%
\sum_i -y_i\alpha_i\ $ s.t. $\ y_i\alpha_i\in[0,1]$& Thm \ref{thm:convBoundSup},\ref{thm:SDCAthm2}
\\ \hline
\end{tabular}
\vspace{2mm}

\setlength\tabcolsep{4pt}
\begin{tabular}{|l|l|lr|} 
\hline
Loss $\ell$ & Logistic Regression &$\ell_{\textsf{LR}}(\zv) := \frac1m\sum_{j=1}^m \log (1\!+\!\exp(-y_j \zv_j))$ & for $\zv=A\alphav$ or $\zv=-A^\top\wv$
\\\cline{2-4}
&Least Squares & $\ell_{\textsf{LS}}(\zv) := \frac12\|\zv-\bv\|_2^2$ & for $\zv=A\alphav$ or $\zv=-A^\top\wv$
\\\cline{1-4}
\end{tabular}\vspace{-2mm}

 \label{tbl:MLProblems}
\end{table*}

{\bf Comparison to smoothing techniques.}
Existing techniques for bringing $\Lone$-regularized and related problems into a primal-dual setting compatible with SDCA do rely on the classical Nesterov-smoothing approach %
- By adding a small amount of $\Ltwo$ to the part $g$, the objective becomes strongly convex; see, e.g.,~\citep{Nesterov:2005ic,nesterov2012efficiency,TranDinh:2014vx,
richtarik2014iteration,ShalevShwartz:2014dy}. %
However, the appropriate strength of smoothing is difficult to tune. It will depend on the accuracy level, and will influence the chosen algorithm, as it will change the iterates, the resulting convergence rate as well as the tightness of the resulting duality gaps.
The line of work of \citet{TranDinh:2014vx,TranDinh:2015wd} provides duality gaps for smoothed problems and rates for proximally tractable objectives \citep{TranDinh:2015wd} and also objectives with efficient Fenchel-type operator \citep{Yurtsever:2015vd}. %
In contrast, our approach preserves all solutions of the original $\Lone$-optimization --- it leaves the iterate sequences of existing algorithms unchanged, which is desirable in practice, and allows the reusability of existing solvers. We do not assume proximal or Fenchel tractability. %

{\bf Distributed algorithms.}
For $\Lone$-problems exceeding the memory capacity of a single computer, a communication-efficient distributed scheme leveraging this Lipschitzing trick is presented in \cite{Smith:2015ua,Forte:2015wv}.

\vspace{-2mm}
\section{Setup and Primal-Dual Structure}
\label{sec:primaldual}
In this paper, we consider optimization problems of the following primal-dual structure. As we will see, the relationship between primal and dual objectives has many benefits, including computation of the duality gap, which allows us to have certificates for the approximation quality.

We consider the following pair of optimization problems, which are dual\footnote{For a self-contained derivation see Appendix \ref{app:duality}.}
to each other:
\begin{align}
    \label{eq:A}\tag{A}\quad  %
    \min_{\alphav \in \R^{n}} \quad& \Big[ \ \ 
    \bD(\alphav) :=&& f(A\alphav )
    \ +\ g(\alphav) &\!\!\Big],\ 
\\
    \label{eq:B}\tag{B}\quad %
    \min_{\wv \in \R^{d}} \quad& \Big[ \ \ 
    \bP(\wv) :=&& f^*(\wv )
    \ +\ g^*(-A^\top\wv) &\!\!\Big].\ 
\end{align}
The two problems are associated to a given data matrix $A\in\R^{d\times n}$, and the functions $f : \R^d \rightarrow \R$ and $g : \R^n \rightarrow \R$ are allowed to be arbitrary closed convex functions.
Here $\alphav \in \R^n$ and $\wv \in \R^d$ are the respective variable vectors.
The relation of \eqref{eq:A} and \eqref{eq:B} is called  \emph{Fenchel-Rockafellar Duality} where the functions $f^*,g^*$ in formulation~\eqref{eq:B} are defined as the \textit{convex conjugates}\footnote{The conjugate is defined as $h^*(\vv) := \sup_{\uv\in\R^d} \vv^\top \uv - h(\uv)$.}
 of their corresponding counterparts $f,g$ in~\eqref{eq:A}.
The two main powerful features of this general duality structure are first that it includes many more machine learning methods than more traditional duality notions, and secondly that the two problems are fully symmetric, when changing respective roles of~$f$ and~$g$. %
In typical machine learning problems, the two parts typically play the roles of a data-fit (or loss) term as well as a regularization term. As we will see later, those two roles can be swapped, depending on the application.

{\bf Optimality Conditions.}
The first-order optimality conditions for our pair of vectors $\wv\in \R^d, \alphav \in \R^n$ in problems~\eqref{eq:A} and \eqref{eq:B} are given as\vspace{-3mm}

\begin{minipage}{0.44\columnwidth}
  \begin{subequations}
  \begin{align}
 \wv \in&\ \partial f(A\alphav) \label{eq:opt_f} \ ,\\ 
A\alphav \in&\ \partial f^*(\wv) \label{eq:opt_fstar} \ ,
\end{align}
  \end{subequations}
\end{minipage}%
\hfill
\begin{minipage}{0.56\columnwidth}
  \begin{subequations}
\begin{align}
- A^\top \wv \in&\ \partial g(\alphav) \label{eq:opt_g} \ ,\\ 
\alphav \in&\ \partial g^*(-A^\top\wv) \!\!\!\!\!\!\!\!\!\!&\label{eq:opt_gstar}
\end{align}
  \end{subequations}
\end{minipage}

see, e.g. \citep[Proposition 19.18]{Bauschke:2011ik}.
The stated optimality conditions are equivalent to $\alphav,\wv$ being a saddle-point of the Lagrangian, which is given as $\mathcal{L}(\alphav,\wv) = f^*(\wv) - \langle A\alphav,\wv\rangle - g(\alphav)$ if $\alphav\in \dom(g)$ and $\wv \in \dom(f^*)$. %

{\bf Duality Gap.}
From the definition of the dual problems in terms of the convex conjugates, we always have $\bP(\wv)\geq\bP(\wv^\star)\geq -\bD(\alphav^\star) \geq - \bD(\alphav) $, giving rise to the definition of the general duality gap 
$G(\wv,\alphav) := \bP(\wv) - (-\bD(\alphav)) $.

For differentiable $f$, the duality gap can be used more conveniently:
Given $\alphav \in \R^{n}$ s.t. $A\alphav\in \dom(f)$ in the context of~\eqref{eq:A}, a corresponding variable vector $\wv\in \R^d$ for problem~\eqref{eq:B} is given by the first-order optimality condition~\eqref{eq:opt_f} as\vspace{-2mm}
\begin{equation}
\label{eq:primaldualrelation}
\wv = \wv(\alphav) := \nabla f( A\alphav ) \, .
\end{equation}
\vspace{-3mm}
\newline
Under strong duality, we have $\bP(\wv^\star)=-\bD(\alphav^\star)$ and $\wv(\alphav^\star) = \wv^\star$, where $\alphav^\star$ is an optimal solution of \eqref{eq:A}. 
This implies that the suboptimality 
$\bP(\wv(\alphav))-\bP(\wv^\star)$ is always bounded above by the simpler duality gap function
\begin{align}
\label{eq:gap}
G(\alphav) \hskip-2pt:=&\bP(\wv(\alphav))\hskip-2pt-\hskip-2pt(-\bD(\alphav)) \hskip-2pt\geq \bP(\wv(\alphav))-\bP(\wv^\star) 
\end{align}
which hence acts as a certificate of the approximation quality of the variable 
vector $\alphav$. 

\section{Primal-Dual Guarantees for Any Algorithm Solving \eqref{eq:A}}

In this section we state an important lemma, which will later allow us to transform a suboptimality guarantee of any algorithm into a duality gap guarantee, for optimization problems of the form specified in the previous section.

\begin{lemma}
\label{lemma:dualityGapAndDualSuboptimality}
Consider an optimization problem of the form~\eqref{eq:A}.
Let $f$ be $1/\beta$-smooth w.r.t. a norm $\|.\|_f$ and let $g$ be $\mu$-strongly convex with convexity parameter $\mu \geq 0$ w.r.t. a norm $\|.\|_g$. The general convex case $\mu=0$ is explicitly allowed, but only if $g$ has bounded support.

Then, for any %
$\alphav \in \dom(\bD)$  and any $s\in [0,1]$, it holds that
\begin{align}\label{eq:dualityGapAndDualSuboptimality} 
 \bD(\alphav)-\bD(\alphav^\star)
&\geq s
G(\alphav)
\\ &  +
\tfrac{s^2}2
\big(
\tfrac{\mu(1-s) }{s}
  \|\uv-\alphav\|_g^2
    - 
 \tfrac1\beta  
\|   A   (\uv-\alphav) \|_f^2
\big),\nonumber 
\end{align}
where $G(\alphav)$ is the gap function defined in 
\eqref{eq:gap} and
\begin{equation}
\label{asdfafsafsa}
\uv \in \partial g^*( -A^\top \wv(\alphav)).\vspace{-1mm}
\end{equation}
\end{lemma}

We note that the improvement bound here bears similarity to the proof of \citep[Prop 4.2]{Bach:2015bz} for the case of an extended Frank-Wolfe algorithm. In contrast, our result here is algorithm-independent, and leads to tighter results due to the more careful choice of $\uv$, as we'll see in the following.

\subsection{Linear Convergence Rates} %

In this section we assume that we are using an arbitrary optimization algorithm applied to problem~\eqref{eq:A}.
It is assumed that the algorithm produces a sequence of (possibly random) iterates
$\{\vc{\alphav}{t}\}_{t=0}^\infty$
such that there exists $C \in (0,1], D \geq 0$ such that
\begin{equation}
\label{eq:afcewwa}
 \Exp[\bD(\vc{\alphav}{t})- \bD(\alphav^\star) ]
  \leq (1-C)^t \ D.
\end{equation}
In the next two theorems, we define $\sigma := \big(\max_{\alphav\ne0} \|A\alphav\|_f/\|\alphav\|_g\big)^2$, i.e., the squared spectral norm of the matrix $A$ in the Euclidean norm case.
\vspace{-1mm}

\subsubsection{Case I. Strongly convex $g$}
Let us assume $g$ is $\mu$-strongly convex ($\mu>0$) (equivalently, its conjugate $g^*$ has Lipschitz continuous gradient with a constant $1/\mu$).
The following theorem provides a linear convergence guarantee for any algorithm with given linear convergence rate for the suboptimality $ \bD(\alphav)-\bD(\alphav^\star)$.
\begin{theorem}
\label{thm:convFastSCg}
Assume the function $f$ is $1/\beta$-smooth w.r.t. a norm $\|.\|_f$ and $g$ is $\mu$-strongly convex w.r.t. a norm $\|.\|_g$.
Suppose we are using a linearly convergent algorithm as specified in \eqref{eq:afcewwa}.
Then, for any 
\begin{equation}\label{eqafsdfasdfas}
 t \geq T := \tfrac1C  \log \tfrac{D(\frac \sigma \beta+\mu)}{\mu \epsilon} 
\end{equation}
it holds that
$ \Exp[G(\vc{\alphav}{t})] \leq \epsilon$.\vspace{-1mm}
\end{theorem}

From \eqref{eq:afcewwa}
we can obtain that after
$\tfrac1C\log \tfrac{D}{\epsilon}$ iterations, we would have a point $\vc{\alphav}{t}$ such that $\Exp[D(\vc{\alphav}{t})-D( \alphav^\star)]\leq \epsilon$.
Hence, comparing with \eqref{eqafsdfasdfas}
only few more iterations are needed to get the guarantees for the duality gap.
The rate~\eqref{eq:afcewwa} 
is achieved by most of the first order algorithms, including
proximal gradient descent \cite{nesterov2013gradient}
or SDCA \cite{richtarik2014iteration}
with  $C \sim  \mu $ 
or 
accelerated SDCA \cite{lin2014accelerated}
with $C \sim \sqrt{ \mu}$.  \vspace{-1mm}

\subsubsection{Case II. General Convex $g$ (Of bounded support)}

In this section we will assume that $g^*$ is Lipschitz (in contrast to smooth as in Theorem \ref{thm:convFastSCg}) and show that the linear convergence rate is preserved.
\begin{theorem}
\label{thm:convFastLipsch}
Assume that the function  $f$ is $1/\beta$-smooth w.r.t. a norm $\|.\|$, $g^*$ is $L$-Lipschitz continuous w.r.t the dual norm $\|.\|_*$, and we are using a linearly convergent algorithm
\eqref{eq:afcewwa}.
Then, for any
\begin{equation}
\label{eq:fsarwjvlawjs} 
t \geq T
:=\tfrac1C  \log    \tfrac{2D
\max\{1,  2  \sigma L^2 / \epsilon \beta \}
}{ \epsilon}
\end{equation}
it holds that
$ \Exp[G(\vc{\alphav}{t})] \leq \epsilon$.
\end{theorem}

In \cite{wang2014iteration}, it was proven that 
feasible descent methods when applied to the dual of an SVM do improve the objective geometrically as in \eqref{eq:afcewwa}.
Later, \cite{ma2015linear} extended this to stochastic coordinate feasible descent algorithms (including SDCA).
Using our new Theorem \ref{thm:convFastLipsch}, we can therefore extend their results to linear convergence for the duality gap for the SVM application.

\subsection{Sub-Linear Convergence Rates} %

In this case we will focus only on general $L$-Lipschitz continuous functions $g^*$ %
(if $g$ is strongly convex, then many existing algorithms are available and converge with a linear rate).
\\
We will assume that we are applying some (possibly randomized) algorithm on optimization problem \eqref{eq:A} which produces a sequence (of possibly random) iterates $\{\vc{\alphav}{t}\}_{t=0}^\infty$ such that
\begin{equation}
 \label{eq:sublinearAlgorithm}
 \Exp[ \bD(\vc{\alphav}{t})- \bD(\alphav^\star)] 
 \leq \tfrac{C}{D(t)},
\end{equation}
where $D(t)$ is a function wich has usually a linear or quadratic growth (i.e. $D(t)\sim \mathcal{O}(t)$ or $D(t)\sim \mathcal{O}(t^2)$).

The following theorem will allow to equip existing algorithms with sub-linear convergence in suboptimality, as specified in ~\eqref{eq:sublinearAlgorithm}, with duality gap convergence guarantees.

\begin{theorem}
\label{thm:convBoundSup}
Assume the function $f$ is $1/\beta$-smooth w.r.t. the norm $\|.\|$, $g^*$ is $L$-Lipschitz continuous, w.r.t. the dual norm $\|.\|_*$,
and we are using a sub-linearly convergent algorithm as quantified by \eqref{eq:sublinearAlgorithm}.
Then, for any $t \geq 0$ such that
\begin{align}\label{eq:asfdwafeawf}
D(t)&\geq  \max\{ \tfrac{2C\beta}{\sigma L^2},   
 \tfrac{2 C \sigma L^2}{\beta  \epsilon^2}\},   
\end{align}
it holds that
$ \Exp[G(\vc{\alphav}{t})] \leq \epsilon$.\vspace{-2mm}
\end{theorem}
Let us comment on Theorem  \ref{thm:convBoundSup} stated above.
If $D(t)\sim \mathcal{O}(t)$ then this shows a rate of $\mathcal{O}{(\epsilon^{-2})}$.
We note two important facts:
\begin{enumerate}[,nolistsep]
 \item The guarantee holds for the duality gap of the iterate $\vc{\alpha}{t}$ and not for some averaged solution.
 \item For the SVM case, this rate is consistent with the result of \cite{hush2006qp}.
 Our result is much more general as it holds for any $L$-Lipschitz continuous convex function $g^*$ and any $\beta$-strongly convex $f^*$.
\end{enumerate}

Let us make one more important remark.
In \cite{ShalevShwartz:2013wl} the authors showed that SDCA (when applied on $L$-Lipschitz continuous $g^*$) has $D(t)\sim \mathcal{O}(t)$ and they also showed that an averaged solution (over few last iterates) needs only $\mathcal{O}(\epsilon^{-1})$ 
 iterations to have duality gap $\leq \epsilon$.
However, as a direct consequence of our Theorem~\ref{thm:convBoundSup} we can show, e.g., that
FISTA \cite{beck2009fast} (aka. accelerated gradient descent algorithm)
or APPROX \cite{Fercoq:2015kd}\footnote{APPROX requires $g$ being separable.} (a.k.a. accelerated coordinate descent algorithm)
will need $\mathcal{O}(\epsilon^{-1})$ iterations to produce an iterate $\alphav$ such that  $G(\alphav)\leq \epsilon$.
Indeed, e.g., for APPROX, the function 
$D(t) = ((t-1)\tau + 2n)^2$ (where $\tau$ is the size of a mini-batch) and $C$ is a constant which depends on $\vc{\alphav}{0}$ and $\alphav^\star$ and $\tau$.
Hence, to obtain an iterate $\vc{\alphav}{t}$ such that $\Exp[G(\vc{\alphav}{t})]\leq \epsilon$ it is sufficient to choose $\tau\geq 1$ such that $t  \overset{\eqref{eq:asfdwafeawf}}{\geq}
 1-\tfrac{2n}{\tau}
 +\max\{
  \tfrac1{\tau}\sqrt{
 \tfrac{2C\beta}{\sigma L^2}
 }, \tfrac{L}{\epsilon \tau}
\sqrt{  \tfrac{2 C\sigma}\beta  }
 \} $ is satisfied.

\section{Extending Duality to Non-Lipschitz Functions}

\subsection{Lipschitzing Trick}
\label{lipschTrick}
In this section we present a trick that allows to generalize our results of the previous section from Lipschitz functions~$g^*$ to non-Lipschitz functions. 
The approach we propose, which we call the \emph{Lipschitzing trick}, will make a convex function Lipschitz on the entire domain~$\R^d$, by modifying its conjugate dual to have bounded support. Formally, the modification is as follows: Let $\bB\subset \R^d$ be some closed, bounded convex set. We modify $g:\R^n\rightarrow\R$ by restricting its support to the set $\bB$, i.e.
  \begin{equation}
\bar g(\alphav) := \begin{cases}
g(\alphav) & \text{if } \alphav\in \bB\\
+\infty & \text{ otherwise}
\end{cases}.
\label{eq:Lip}
\end{equation}
By definition, this function has bounded support, and hence, by Lemma \ref{lem:dualLipschitz}, its conjugate function $\bar g^*$ is Lipschitz continuous.

{\bf Motivation.}
We will apply this trick to the part $g$ of optimization problems of the form \eqref{eq:A} (such that $g^*$ will become Lipschitz). 
We want that this modification to  have no impact on the outcome of any optimization algorithm running on~\eqref{eq:A}. Instead, this trick should only affect convergence theory in that it allows us to present a strong primal-dual rate. 
In the following, we will discuss how this can indeed be achieved by imposing only very weak assumptions on the original problem~\eqref{eq:A}.

The modification is based on the following duality of Lipschitzness and bounded support, as given in Lemma~\ref{lem:dualLipschitz} below, which is a generalization of~\citep[Corollary 13.3.3]{Rockafellar:1997ww}. We need the following definition:

\begin{definition}[$B$-Bounded Support]
\label{def:lbounded}
A function $g: \R^d \to \R$ has \emph{$B$-bounded support} if its effective domain is bounded by $B$ w.r.t. a norm $\|.\|$, i.e.,
\begin{equation}
  g(\uv) < + \infty  \ \Rightarrow \  \|\uv\| \le B \, .
\end{equation}
\end{definition}

\begin{lemma}[{Duality between Lipschitzness and L-Bounded Support}]
\label{lem:dualLipschitz}
Given a proper convex function $g$, it holds that
$g$ has $L$-bounded support w.r.t. the norm $\|.\|$ if and only if 
$g^*$ is $L$-Lipschitz w.r.t. the dual norm $\|.\|_*$.
\end{lemma}
\vspace{-2mm}

The following Theorem \ref{thm:convGeneral} generalizes our previous convergence results, which were restricted to Lipschitz $g^*$.

\begin{theorem}
\label{thm:convGeneral}

For an arbitrary optimization algorithm running on problem~\eqref{eq:A}, let $\vc\alphav t$ denote its iterate sequence. Assume there is some closed convex set $\bB$ containing all these iterates. Then, the same optimization algorithm run on the modified problem --- given by Lipschitzing of $g^*$ using $\bB$ --- would produce exactly the same iterate sequence.
Furthermore,  Theorem~\ref{thm:convFastLipsch} as well as Theorem~\ref{thm:convBoundSup} give primal-dual convergence guarantees for this algorithm (for $L$ such that  $\bB$ is  $L$-bounded).
\end{theorem}

\begin{corollary} Assume the objective of optimization problem~\eqref{eq:A} has bounded level sets. For $\vc\alphav t$ being the iterate sequence of a \emph{monotone} optimization algorithm on problem~\eqref{eq:A} we denote $\delta_t := \bD(\vc\alphav t)$ and let $\bB_t$ be the $\delta_t$-level set of $\bD$. Write $B_t>0$ for a value such that $\bB_t$ is $B_t$-bounded w.r.t. the norm $\|.\|$. Then, at any state $t$ of the algorithm, the set $\bB_t$ contains all future iterates and Theorem \ref{thm:convGeneral} applies for $L:=B_t$.
\vspace{-2mm}

\end{corollary}

\begin{figure}[t]
\centering
\begin{tikzpicture}[xscale=3,yscale=3]
\draw [->] (-1,0) -- node[above, at end]{$g(\alpha)$}(-1,0.5);
\draw [->] (-1.5,0) -- node[right, at end]{$\alpha$}(-0.5,0);
\draw [thick](-1,0) -- (-1.5,0.5);
\draw [thick](-1,0) -- (-0.5,0.5);
\draw [thick, dashed, blue](-1,0) -- (-1.4,0.4);
\draw [dashed, blue,thick](-1,0) -- (-0.6,0.4);
\draw[dashed,thin](-1.4,0.4) -- node[below, at end, black]{$-B$}(-1.4,0);
\draw[dashed,thick, blue](-1.4,0.4) -- (-1.4,0.6);
\draw[dashed,thin](-0.6,0.4) -- node[below, at end, black]{$B$}(-0.6,0);
\draw[dashed,thick, blue](-0.6,0.4) -- (-0.6,0.6);
\draw [->] (0.5,0) -- node[above, at end]{$g^*(u)$}(0.5,0.5);
\draw [->] (0,0) -- node[right, at end]{$u$}(1,0);
\draw [thick](0.8,0) -- (0.8,0.5);
\draw [thick](0.2,0) -- (0.2,0.5);
\draw [thick](0.2,0) -- (0.8,0);
\draw [dashed, blue,thick](0.2,0) -- (0.8,0);
\draw [dashed,thick, blue](0.8,0) -- (0.9,0.5);
\draw [dashed,thick, blue] (0.2,0) -- (0.1,0.5);
\draw [white](-0.2,0) -- node[black,above, at end]{$\large\xrightarrow{()^*}$}(-0.2,0.2);
\end{tikzpicture}
\caption[Illustration of $\theta$-orthogonality]{Illustration Lipschitzing trick for the scalar function $g(\alpha)=|\cdot|$ with $\bB=\{x: |x|\leq B\}$.}
\label{fig:LipschitzingTrick}
\vskip-0.4cm
\end{figure}
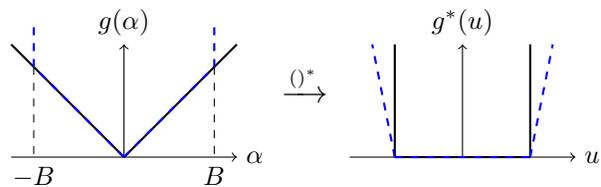
\vspace{-0.1cm}
\subsection{Norm-Regularized Problems}
\label{sec:normreg}

We now focus on some applications. First, we demonstrate how the Lipschitzing trick can be applied to find primal-dual convergence rates for problems regularized by an arbitrary norm. We discuss in particular the Lasso problem and show how the  suboptimality bound can be evaluated in practice.
In a second part, we discuss the Elastic Net regularizer and show how it fits into our framework.

We consider a special structure of problem \eqref{eq:A}, namely 
\begin{equation}\label{eq:NormReg}
\min_{\alphav\in\R^n} f(A\alphav)+ \lambda \norm {\alphav}.
\end{equation}
where $f$ is some convex non-negative smooth loss function regularized by any norm $\|.\|$. We choose $\bB$ to be the $\|.\|$-norm ball of radius $B$, such that  $\|\alphav\|< B$ for every iterate. The size, $B$, of this ball can be controled by the amount of regularization $\lambda$. 
Using the Lipschitzing trick with this $\bB$, the convergence guarantees of Theorem \ref{thm:convGeneral} apply to the general class of norm regularized problems.

Note that for any monotone  algorithm (initialized at $\0$) applied to \eqref{eq:NormReg} we have $\|\alphav\|\leq \tfrac 1\lambda f(\0)$ for every iterate. Furthermore, at every iterate $\alphav$ we can bound $\|\alphav^+\|\leq\tfrac1 {\lambda} ( f(A\alphav)+ \lambda \norm {\alphav})$ for every future iterate $\alphav^+$. Hence, $B:=\tfrac 1\lambda f(\0)$ is a safe choice, e.g. $B:= \tfrac{1}{2 \lambda} \|\bv\|_2^2$ for least squares loss and  $B:= \tfrac{m}{\lambda} \log(2)$ for logistic regression loss.

{\bf Duality Gap.} For any problem of the form \eqref{eq:NormReg} we can now determine the duality gap.
We apply the Lipschitzing trick to $g(\alphav):=\lambda\|\alphav\|$ as in \eqref{eq:Lip}, then the convex conjugate of $\bar g$ is\vspace{-2mm}
\begin{equation}\bar g^*(\uv) = 
\begin{cases}
0 & \|\uv\|_*\leq \lambda\\
\underset{\alphav:  \alphav \in \bB}{\max} \uv^\top \alphav - \lambda\|\alphav\|& \text{else}
\end{cases}.
\label{eq:barg*}
\end{equation}
where $\|.\|_*$ denotes the dual norm of $\|.\|$. Hence, using the primal-dual mapping $\wv(\alphav):=\nabla f(A \alphav)$ we can write the duality gap of the modified problem  as
\begin{equation}
\label{eq:gapLip}
\bar G(\alphav)=\langle \wv(\alphav), A\alphav\rangle+ \lambda \|\alphav\| + \bar g^*(-A^\top\wv(\alphav))
\end{equation}
Note that the computation of the modified duality gap $\bar G(\alphav)$ is not harder than the computation of the original gap $G(\alphav)$ - it requires one pass over the data, assuming the choice of a suitable set $\bB$ in \eqref{eq:barg*}.
Furthermore,  note that  in contrast to the unmodified duality gap $G(\alphav)$, which is only defined on the set $\left\{\alphav: \|A^\top\wv(\alphav)\|_*\leq \lambda \right\}$, our new gap $\bar G(\alphav)$ is defined on the entire space~$\R^d$.

As an alternative, \citep[Sec. 1.4.1 and App. D.2]{Mairal:2010vm} has defined a different duality gap by shrinking the dual $\wv$ variable until it becomes feasible in $\|A^\top\wv\|_*\leq \lambda$. %

\subsubsection{$L_1$-Regularized Problems}
\label{sec:Lasso}
The results from the previous section can be ported to the important special case of $L_1$-regularization:
\begin{equation}\label{eq:LassoObj}
\min_{\alphav\in\R^n} f(A\alphav)+ \lambda \norm {\alphav}_1.
\end{equation}
We choose $\bB$ to be the $L_1$-norm ball of radius $B$ and then apply the Lipschitzing trick with $\bB$ to the regularization term in \eqref{eq:LassoObj}. An illustration of this modification as well as the impact on its dual are illustrated in Figure \ref{fig:LipschitzingTrick}.
Hence, our theory (Theorem \ref{thm:convGeneral}) gives primal-dual convergence guarantees for any algorithm applied to the Lasso problem~\eqref{eq:LassoObj}. Furthermore, if the algorithm is monotone (initialized at $\alphav:=\0$) we know that $\bB$ is $\tfrac 1\lambda f(\0)$-bounded.

{\bf Duality Gap.}
The duality gap \eqref{eq:gapLip} for the modified Lasso problem can be computed at every iterate $\alphav$ as
\begin{equation}
\bar G(\alphav) =\ \langle \wv,A \alphav\rangle + B \left[\| A^\top\wv\|_\infty-\lambda\right]_+ +\lambda \|\alphav\|_1\label{eq:GLasso}
\end{equation}
for $\wv=\wv(\alphav)$. For the derivation, see Appendix  \ref{sec:DualityGapLasso}.

\subsubsection{Group Lasso, Fused Lasso and TV-Norm}
Our algorithm-independent, primal-dual convergence guarantees and certificates presented so far are not restricted to $L_p$-regularized problems, but do in fact directly apply to many more general structured regularizers, some of them shown in see Table~\ref{tbl:MLProblems}.
This includes group Lasso ($L_1/L_2$) and other norms inducing structured sparsity \cite{Bach:2012fe}, as well as other penalties such as e.g. the fused Lasso $g(\alphav) := \lambda \|M\alphav\|_1$.
The total variation denoising problem is obtained for suitable choice of the matrix $M$.

\subsubsection{Elastic Net Regularized Problems}
\label{sec:ElasticNet}
The second application we will discuss  is Elastic Net regularization
\begin{equation}\label{eq:elasticNet}
\min_{\alphav \in \R^n} \ell(A\alphav) +\lambda \left(\frac{\eta}{2} \|\alphav\|_2^2 + (1-\eta)\|\alphav\|_1\right),
\end{equation}
for fixed trade-off parameter $\eta\in(0,1]$.

Our framework  allows two different ways to solve problem~\eqref{eq:elasticNet}: Either mapping it to formulation \eqref{eq:A} (for $\ell$=$f$ smooth) as in row~2 of  Table \ref{tbl:MLProblems}, or to \eqref{eq:B} (for general $\ell$) as in row~3 of Table~\ref{tbl:MLProblems}. 
In both scenarios, Theorem \ref{thm:convFastSCg} gives us a fast linear convergence guarantee for the duality gap, if~$\ell$ is smooth. The other theorems apply accordingly for general~$\ell$ when the problem is mapped to \eqref{eq:B}.
\\
Whether the choice of a dual or primal optimizer will be more beneficial in practice depends on the case, and will be discussed in more detail for coordinate descent methods in Section \ref{sec:CD}.

{\bf Duality Gap.}
For the Elastic Net problem \eqref{eq:elasticNet} mapped to~\eqref{eq:A}, we can compute the duality gap \eqref{eq:gap} as follows:
\begin{align*}
G(\alphav) &=\langle\wv, A\alphav\rangle+\tfrac{1}{2\eta \lambda}\sum_{i=1}^n  \left[\left|A_{:i}^\top\wv\right|-(1-\eta)\lambda\right]_+^2\\
&+\lambda \left(\tfrac{\eta}{2} \|\alphav\|_2^2 + (1-\eta)\|\alphav\|_1\right)
\end{align*}
with $\wv=\wv(\alphav) = \nabla \ell(A \alphav)$, see Appendix \ref{sec:DualityGapElastic}. %
\begin{remark}
As $\eta\rightarrow0$ we approach the pure $L_1$-case and  this gap blows up as $G(\alphav)\rightarrow\infty$. Comparing this to \eqref{eq:GLasso}, we see that the Lipschitzing trick  allows to get certificates even in cases where the duality gap of the unmodified problem is infinity.
\end{remark}

\section{Coordinate Descent Algorithms}
\label{sec:CD}
We now focus on a very important class of algorithms, that is coordinate descent methods.
In this section, we show how our theory implies much more general primal-dual convergence guarantees for coordinate descent algorithms.

{\bf Partially Separable Problems.}
A widely used subclass of optimization problems arises when one part of the objective becomes separable. 
Formally, this is expressed as $g(\alphav) = \sum_{i=1}^{n} g_i(\alpha_i)$ for univariate functions $g_i:\R\rightarrow\R$ for $i\in[n]$.
Nicely in this case, the conjugate of $g$ also separates as
$g^*(\yv) = \sum_i g_i^*(y_i)$. Therefore, the two optimization problems \eqref{eq:A} and \eqref{eq:B} write as
\begin{align}
    \bD(\alphav) :=&\ f(A\alphav )
    + \textstyle\sum_i g_i(\alpha_i) \label{eq:Dsep}\tag{SA}\\
    \bP(\wv) :=&\ f^*(\wv )
    + \textstyle\sum_i g_i^*(-A_{:i}^{\ \top}\wv) \ ,\label{eq:Psep}\tag{SB}\
\end{align}
where $A_{:i} \in \R^d$ denotes the $i$-th column of $A$.

\paragraph{The Algorithm.}
We consider the coordinate descent algorithm described in Algorithm \ref{alg:SDCA}. Initialize $\vc{\alphav}{0}=\0$ and then, at each iteration, sample and update a random coordinate $i\in[n]$ of the parameter vector $\alphav$ to iteratively minimize \eqref{eq:Dsep}. Finally, after $T$ iterations output $\bar \alphav$, the average vector over the latest $T-T_0$ iterates. The parameter $T_0$ is some positive number smaller than $T$. 

\setlength{\textfloatsep}{15pt}
\begin{algorithm}[h]
\caption{Coordinate Descent on $\bD(\alphav)$}
\label{alg:SDCA}
\begin{algorithmic}[1]
\STATE {\bf Input:} Data matrix $A$.\\
Starting point $\vc{\alphav}{0} := \0 \in \R^n$, $\vc{\wv}{0}=\wv(\vc{\alphav}{0})$.
\FOR {$t =  1, 2, \dots T$}
     \STATE Pick $i\in[n]$ randomly
	  \STATE Find $\Delta\alpha_i$ minimizing $\bD(\vc{\alphav}{t-1}+\ev_i \Delta \alpha_i)$
	  \STATE $\vc{\alphav}{t}\leftarrow \vc{\alphav}{t-1}+\Delta \alpha_i \ev_i$
	  \STATE $\vc{\wv}{t} \leftarrow \wv(\vc{\alphav}{t})$
\ENDFOR 
\STATE Let $\bar{\alphav}= \frac{1}{T-T_0}\sum_{t=T_0}^{T-1}\vc{\alphav}{t}$ 
\end{algorithmic}
\end{algorithm}

As we will show in the following section, coordinate descent on $\bD(\alphav)$ is not only an efficient optimizer of the objective $\bD(\alphav)$, but also provably reduces the duality gap. Therefore, the same algorithm will simultaneously optimize the dual objective $\bP(\wv)$.

\subsection{Primal-Dual Analysis for Coordinate Descent} 
We first show linear primal-dual convergence rate of Algorithm \ref{alg:SDCA} applied to \eqref{eq:Dsep} for strongly convex $g_i$. 
Later, we will generalize this result to also apply to the setting of general Lipschitz $g_i$. This generalization together with the Lipschitzing trick will allow us to derive primal-dual convergence guarantees of coordinate descent for a much broader class of problems, including the Lasso problem.

For the following theorems we assume that the columns of the data matrix $A$ are scaled such that $\|A_{:i}\| \leq R$ for all $i\in[n]$ and $\|A_{j:}\| \leq P$ for all $ j\in[d]$, for some norm $\|.\|$.

\begin{theorem}
\label{thm:SDCAthm5}
Consider Algorithm \ref{alg:SDCA} applied to \eqref{eq:Dsep}. Assume $f$ is a $1/\beta$-smooth function w.r.t. the norm $\|.\|$. Then, if $g_i$ is $\mu$-strongly convex for all~$i$, it suffices to have a total number of iterations of 
\begin{align*}
T&\geq \left(n+\tfrac{n R^2}{\mu \beta}\right) \log\left( \left[n+\tfrac{n R^2}{\mu \beta}\right] \tfrac{\epsilon_D^{(0)}}{\epsilon}\right)
\end{align*}
to get $\Exp[G(\vc{\alphav}{T})] \leq \epsilon$.
Moreover, to obtain an expected duality gap of $\Exp[G(\bar \alphav)] \leq \epsilon$ it suffices to have $T>T_0$ with 
\begin{align*}
T_0&\geq \left(n+\tfrac{n R^2}{\mu \beta}\right) \log\left( \left[n+\tfrac{n R^2}{\mu \beta}\right] \tfrac{\epsilon_D^{(0)}}{(T-T_0)\epsilon}\right)
\end{align*}
where $ \epsilon_D^{(0)}$ is the initial suboptimality in $\bD(\alphav)$.
\end{theorem}

Theorem \ref{thm:SDCAthm5} allows us to upper bound the duality gap, and hence the suboptimality, for every iterate $\vc \alphav T$, as well as the average $\bar \alphav$ returned by Algorithm \ref{alg:SDCA}.
In the following we generalize this result to apply to $L$-Lipschitz functions~$g_i$.

\begin{theorem}%
\label{thm:SDCAthm2}
Consider Algorithm \ref{alg:SDCA} applied to \eqref{eq:Dsep}. Assume $f$ is a $1/\beta$-smooth function w.r.t. the norm $\|.\|$. Then, if $g_i^*$ is $L$-Lipschitz for all $i$,  it suffices to have a total number of iterations of 
\begin{align*}
T&\geq \max\left\{0,n\log\tfrac{ \epsilon_D^{(0)} \beta}{ 2L^2R^{2} n}\right\}+  n +\tfrac{20 n^2 L^2  R^2}{\beta \epsilon }
\end{align*}
to get $\Exp[G(\bar \alphav)] \leq \epsilon$. Moreover, when $t\geq T_0$ with
\[T_0= \max\left\{0,n\log\tfrac{ \epsilon_D^{(0)} \beta}{ 2L^2R^{2} n}\right\}+\tfrac{16 n^2L^2R^2}{ \beta \epsilon}\]
 we have the suboptimality bound of $\Exp[\bD(\alphav^{(t)})-\bD(\alphav^\star)]\leq \epsilon/2$, where $ \epsilon_D^{(0)}$ is the initial suboptimality.
\end{theorem}

\begin{remark}
Theorem \ref{thm:SDCAthm2} shows that for Lipschitz $g_i^*$, Algorithm \ref{alg:SDCA} has $ \mathcal{O}(\epsilon^{-1})$ convergence in the suboptimality and $ \mathcal{O}(\epsilon^{-1})$ convergence in $G(\bar\alphav)$. 
Comparing this result to  Theorem \ref{thm:convBoundSup} which suggests $ \mathcal{O}(\epsilon^{-2})$ convergence in $G(\alphav)$ for $ \mathcal{O}(\epsilon^{-1})$ convergent algorithms, we see that  averaging the parameter vector crucially improves convergence in the case of non-smooth $f$.
\end{remark}

\begin{remark}\label{rem:recoverSDCA}
Note that our Algorithm~\ref{alg:SDCA} recovers the widely used SDCA setting \citep{ShalevShwartz:2013wl}  as a special case, when we choose $f^*:=\frac{\lambda}{2} \|.\|_2^2$ in \eqref{eq:Psep}.
Furthermore, their convergence results for SDCA are consistent with our results and can be recovered as a special case of our analysis. See Corollaries \ref{cor:Thm5SDCA},  \ref{cor:Thm2SDCA}, \ref{cor:Lemma19SDCA} in Appendix \ref{sec:CDproofs}.
\end{remark}

\subsection{Application to $L_1$ and Elastic Net Regularized Problems}

We now apply Algorithm \ref{alg:SDCA} to the $L_1$-regularized problems, as well as Elastic Net regularized problems. We state improved primal-dual convergence rates which are more tailored to the coordinate-wise setting.

{\bf Coordinate Descent on $L_1$-Regularized Problems.}
In contrast to the general analysis of $L_1$-regularized problems in Section \ref{sec:Lasso}, we can now exploit separability of $g(\alphav):=\lambda \|\alphav\|_1$ and apply the Lipschitzing trick coordinate-wise, choosing $\bB:=\left\{\alpha : |\alpha|\leq B\right\} \subset\R$. 
This results in the following stronger convergence results:
\begin{corollary}\label{cor:l1CD}
We can use the Lipschitzing trick together with Theorem \ref{thm:SDCAthm2} to derive a primal-dual convergence result for the Lasso problem \eqref{eq:LassoObj}. We find that the $g_i^*$ are $B$-Lipschitz after applying the Lipschitzing trick to every~$g_i$, and hence the total number of iterations needed on the Lasso problem to get a duality gap of  $\Exp[G(\bar \alphav)] \leq \epsilon$ is 
\begin{align*}
T&\geq \max\left\{0,n\log\tfrac{ \beta \epsilon_D^{(0)}}{2 B^2R^{2} n}\right\}+  n +\tfrac{20 n^2 B^2  R^2}{\beta \epsilon }
\end{align*}
\end{corollary}

\begin{remark} We specify the different parameters of Corollary \ref{cor:l1CD} for least squares loss as well as the logistic regression loss (defined in Table \ref{tbl:MLProblems}). Both are $1$-smooth ($f^*$ is $1$-strongly convex) and we have $\beta:=1$. The initial suboptimality $\vc {\epsilon_D} 0 $ can be upper bounded by  $ \tfrac{1}{2}\|\bv\|_2^2$ for the former  and by $m \log(2)$ for the latter.  
For $B$ we choose $\frac1\lambda f(\0)$.
\end{remark}

{\bf Coordinate Descent on Elastic Net Regularized Problems.}
In Section \ref{sec:ElasticNet} we discussed how the Elastic Net problem in \eqref{eq:elasticNet} can be mapped to our setup. In the first scenario (row~2, Table \ref{tbl:MLProblems}) we note that the resulting problem is partially separable and an instance of \eqref{eq:Dsep}.
\\
In the second scenario we map  \eqref{eq:elasticNet} to $\eqref{eq:B}$ (row~3, Table \ref{tbl:MLProblems}). Assuming that the loss function $\ell$ is separable, this problem is an instance of \eqref{eq:Psep}.  The convergence guarantees when applying Algorithm~\ref{alg:SDCA} on the primal or on the dual are summarized in Corollary~\ref{cor:elasticCD}.

\begin{corollary}\label{cor:elasticCD}
Consider Algorithm \ref{alg:SDCA} for an Elastic Net regularized problem \eqref{eq:elasticNet}, running on either the primal or the dual formulation.
Then, to obtain a duality gap of  $\Exp[G( \vc{\alphav}{T})] \leq \epsilon$, it suffices to have a total of
\begin{align*}
T&\geq  (n+\tfrac{n R^2}{\lambda \eta \zeta} ) \log  ( [n+\tfrac{n R^2}{ \lambda \eta \zeta} ] \tfrac{\epsilon_D^{(0)}}{\epsilon} )
\end{align*}
 iterations for coordinate descent on the primal \eqref{eq:elasticNet} and
\begin{align*}
T&\geq  (d+\tfrac{d P^2}{\lambda \eta \zeta} ) \log (  [d+\tfrac{d P^2}{ \lambda \eta \zeta} ] \tfrac{\epsilon_D^{(0)}}{\epsilon} )
\end{align*}
for coordinate descent on the dual of \eqref{eq:elasticNet}.\vspace{-1mm}
\end{corollary}
According to Corollary \ref{cor:elasticCD}, the convergence rate is comparable for both scenarios. The constants however depend on the data matrix $A$ -- for $d\gg n$ the primal version is beneficial, whereas for $n \gg d$ the dual version is leading.

\section{Numerical Experiments}

Here we illustrate the usefulness of our framework by showcasing it for two important applications, each one showing two algorithm examples for optimizing~\eqref{eq:A}.

{\bf Lasso.}
The top row of Figure \ref{fig:SVM} shows the primal-dual convergence of Algorithm \ref{alg:SDCA} (CD) as well as the accelerated variant of CD (APPROX, \citet{Fercoq:2015kd}), both applied to the Lasso problem~\eqref{eq:A}.
We have applies the Lipschitzing trick as described in Section \ref{lipschTrick}.
This makes sure that $\wv(\alphav)$ will be always feasible for the modified dual~\eqref{eq:B}, 
and hence the duality gap can be evaluated. %

{\bf SVM.} It was shown in \cite{ShalevShwartz:2013wl}
that if CD (SDCA) is run on the dual SVM formulation, and we consider an "average" solution (over last few iterates), then the duality gap evaluated at averaged iterates has a sub-linear convergence rate $\mathcal{O}(1/t)$.
As a consequence of Theorem \ref{thm:convBoundSup}, we have that the APPROX algorithm \citep{Fercoq:2015kd} will provide the same sub-linear convergence in duality gap, but holding for the iterates themselves, not only for an average. %
On the bottom row of Figure \ref{fig:SVM} we compare CD with its accelerated variant on two benchmark datasets.\footnote{Available from  \href{https://www.csie.ntu.edu.tw/~cjlin/libsvmtools/datasets/}{csie.ntu.edu.tw/$\sim$cjlin/libsvmtools/datasets}.} We have chosen $\lambda=1/n$.

\begin{figure}[t]
\centering
 \includegraphics[scale=0.17]{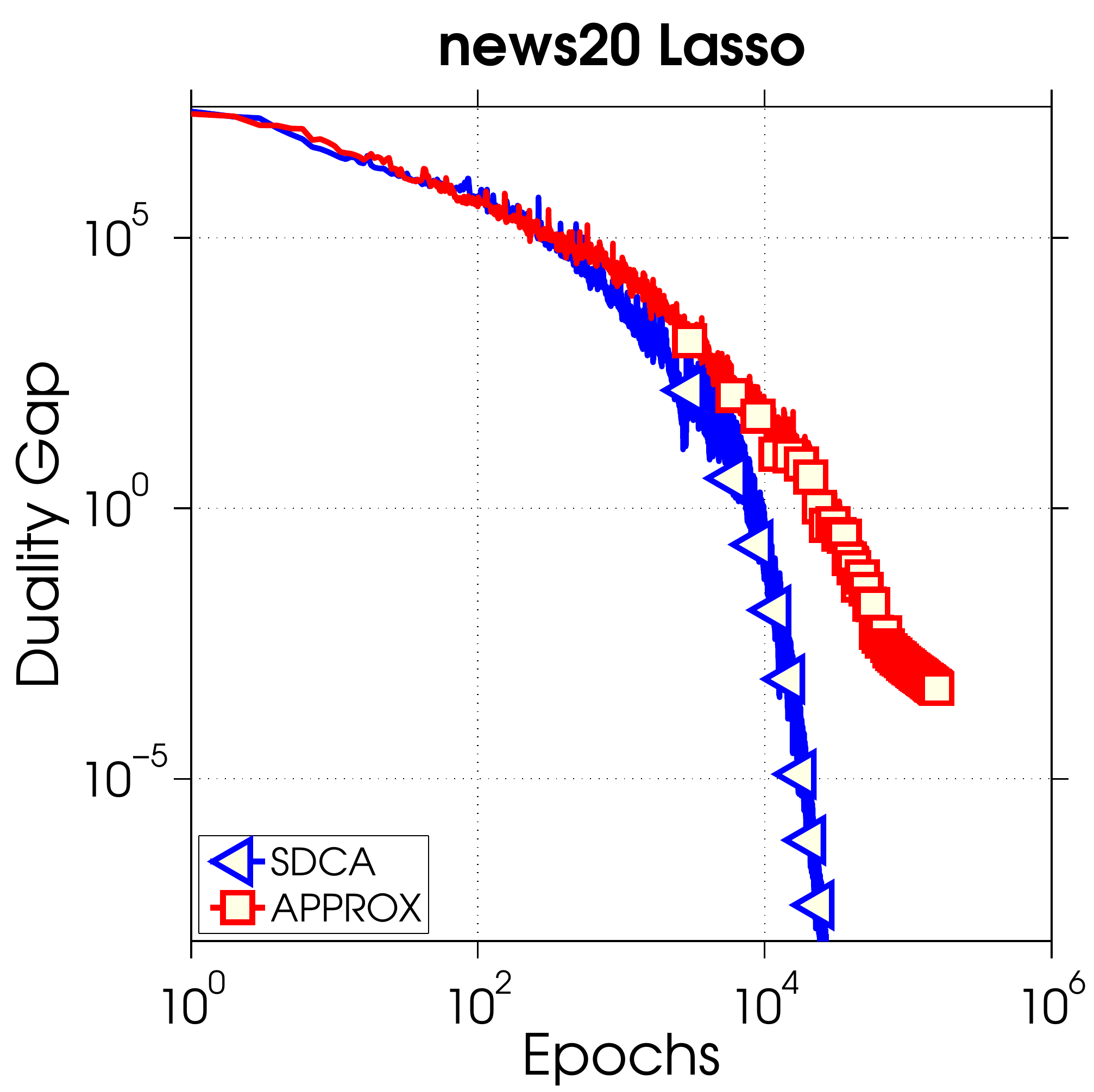} 
\includegraphics[scale=0.17]{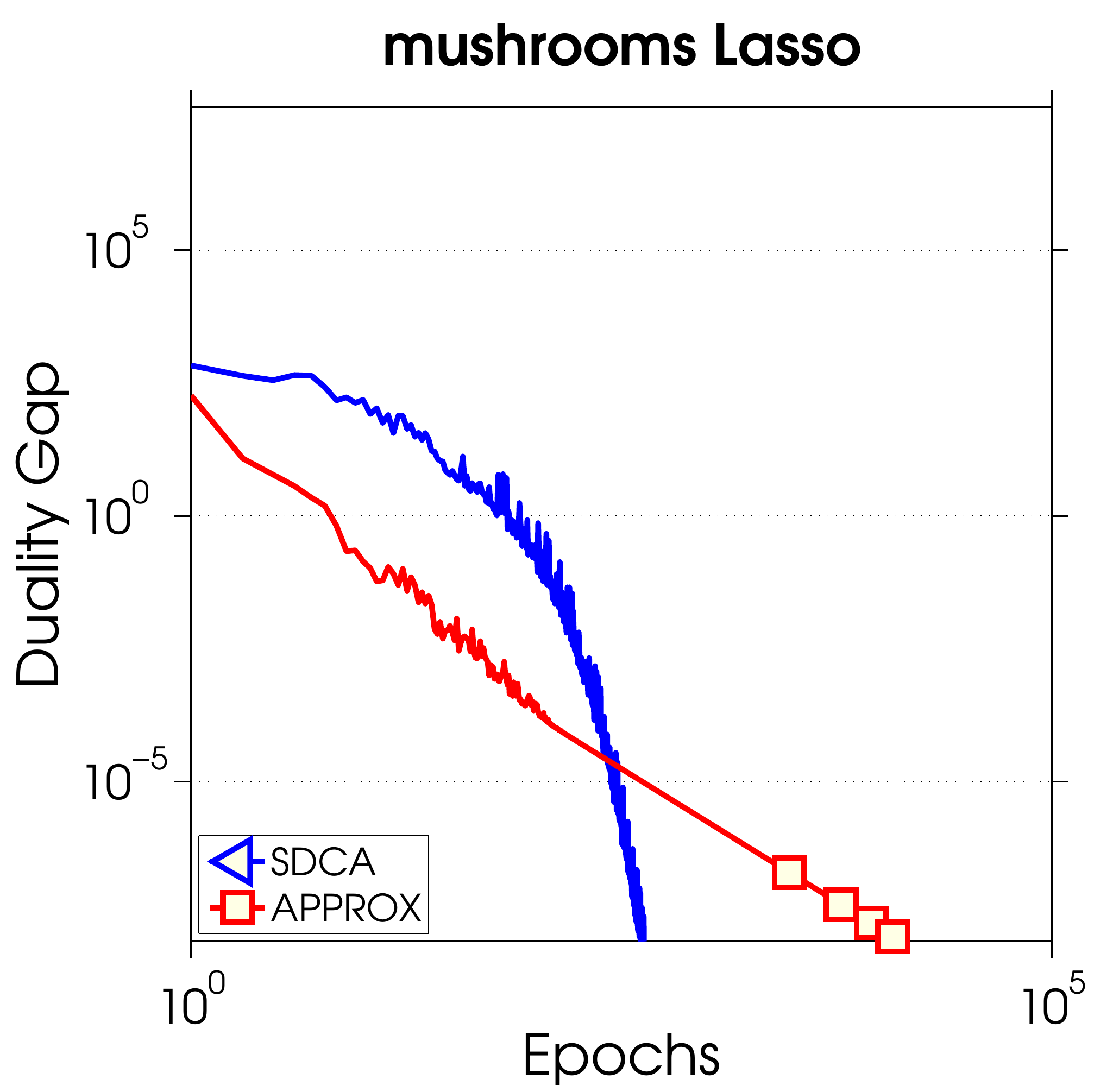}

\includegraphics[scale=0.17]{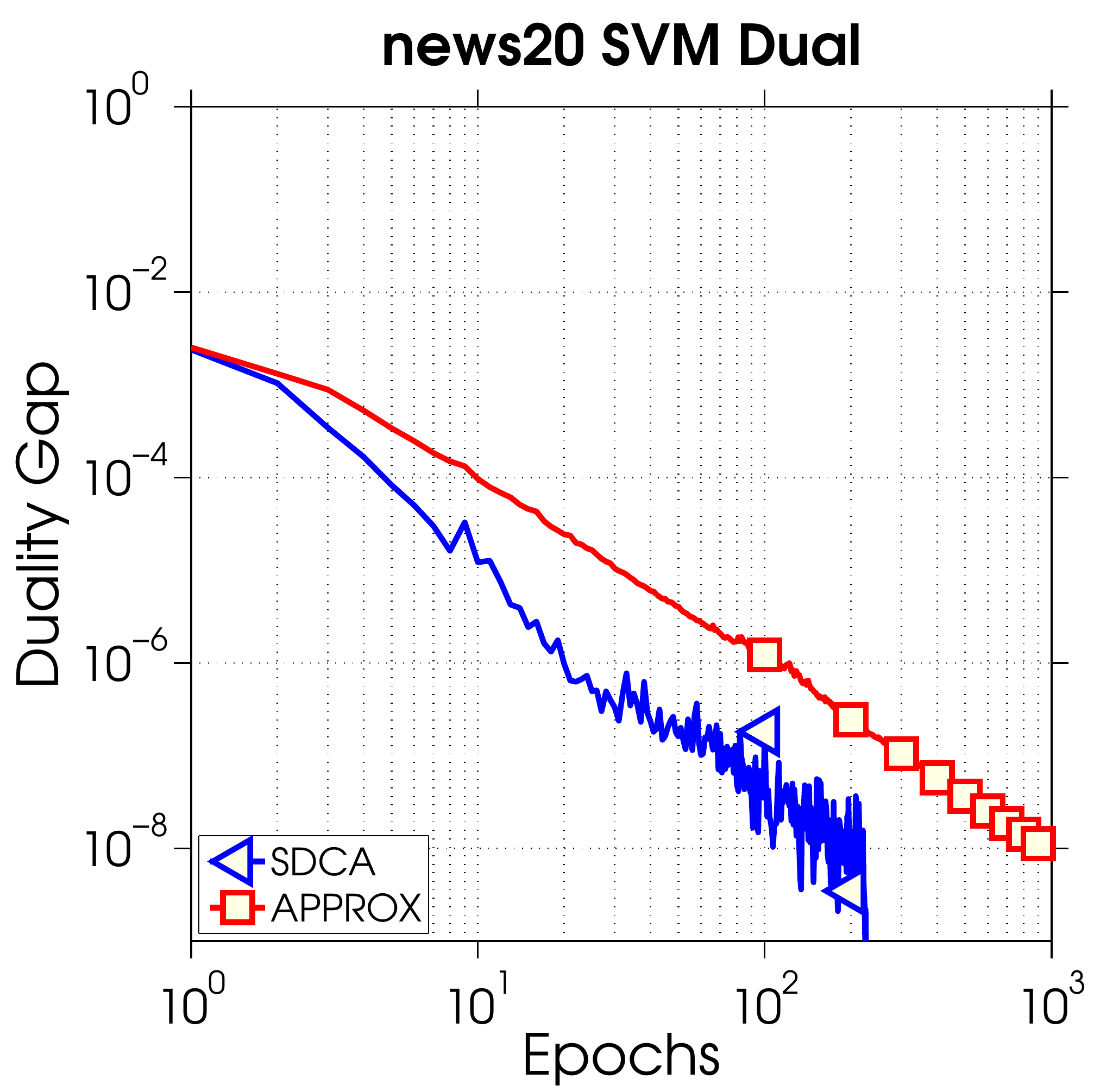}
\includegraphics[scale=0.17]{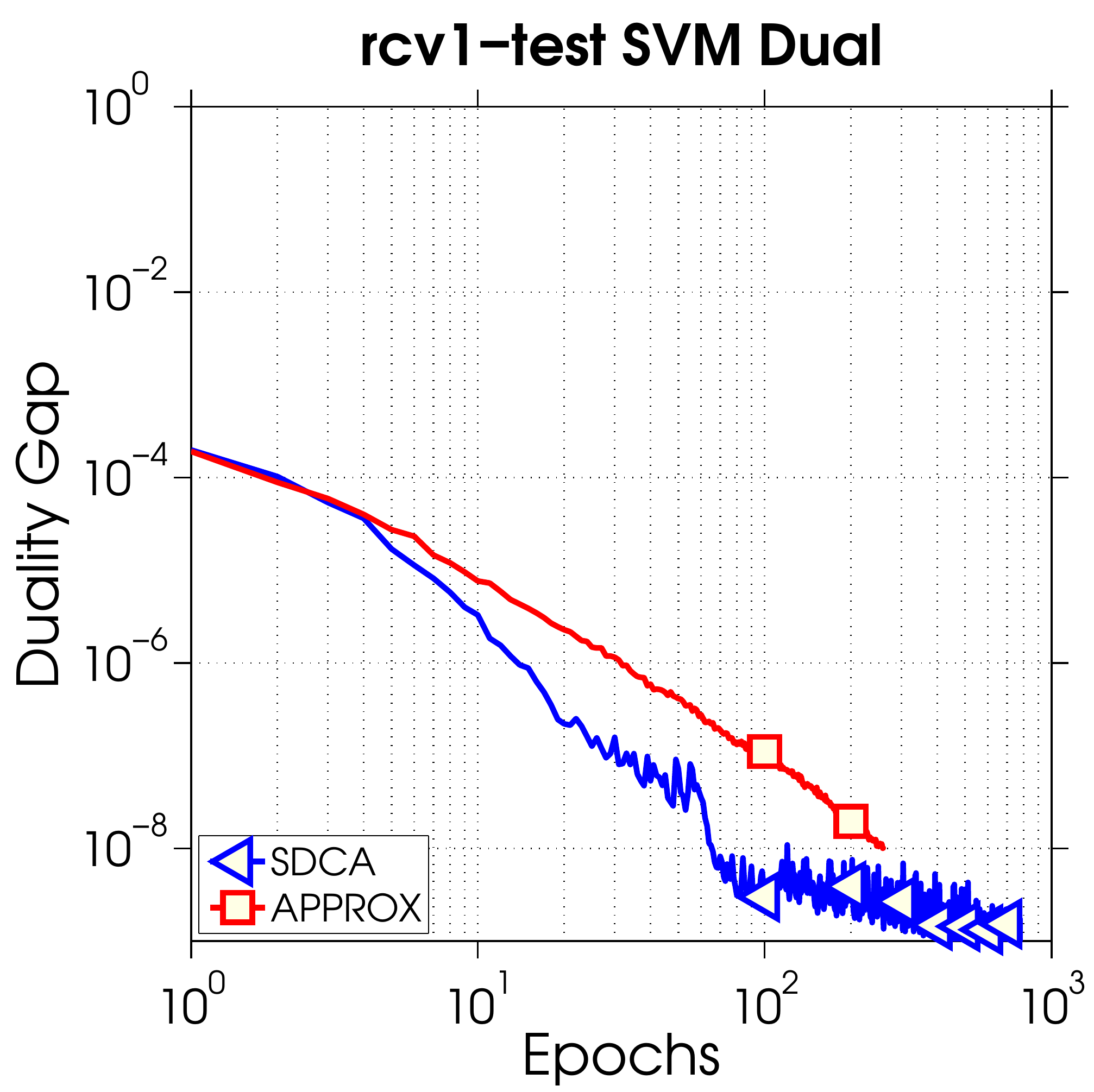}\vspace{-1mm}
\caption{Comparison of CD with its accelerated variant -- APPROX on Lasso and SVM problems.}
\label{fig:SVM} %
\end{figure}

\vspace{-1mm}
\section{Conclusions}
We have presented a general framework allowing to equip existing optimization algorithms with primal-dual certificates.
For future research, it will be interesting to study more applications and algorithms fitting into the studied problem structure, including more cases of structured sparsity, and generalizations to matrix problems. 

\paragraph{Acknowledgments.}
We thank Francis Bach, Michael P. Friedlander, Ching-pei Lee, Dmytro Perekrestenko, Aaditya Ramdas, Virginia Smith and an anonymous reviewer for fruitful discussions.

 \twocolumn
{\small
\bibliographystyle{icml2016}
\bibliography{bibliography}
}
 \onecolumn

\clearpage
\appendix
\setlength{\belowdisplayskip}{5pt} \setlength{\belowdisplayshortskip}{3pt}
\setlength{\abovedisplayskip}{5pt} \setlength{\abovedisplayshortskip}{3pt}
\part*{Appendix}

\section{Basic Definitions}

\begin{definition}[$L$-Lipschitz Continuity]
A function $h: \R^d \to \R$ is \emph{$L$-Lipschitz continuous} w.r.t. a norm $\|.\|$ if $\forall \av,\bv \in \R^d$, we have\vspace{-2mm}
\begin{equation}
 | h(\av) - h(\bv) | \leq L \| \av-\bv \| \, .
\end{equation}
\end{definition}

\begin{repdefinition}{def:lbounded}[$B$-Bounded Support]
A function $h: \R^d \to \R$ has \emph{$B$-bounded support} if its effective domain is bounded by $B$ w.r.t. a norm $\|.\|$, i.e.,
\begin{equation}
  h(\uv) < + \infty  \ \Rightarrow \  \|\uv\| \le L \, .
\end{equation}
\end{repdefinition}

\begin{definition}\label{def:levelset}
The $\delta$-level set of a function $h:\R^d\rightarrow \R$ is defined as $\bL(\delta):=\left\{\xv : h(\xv)\leq \delta \right\}$.
\end{definition}

\begin{definition}[$L$-Smoothness]
A function $h:\R^d\rightarrow\R$ is called \emph{$L$-smooth} w.r.t. a norm $\|.\|$, for $L>0$, if
it is differentiable and its derivative is $L$-Lipschitz continuous w.r.t. $\|.\|$,
or equivalently
\begin{equation}
h(\uv) \leq h(\wv) + \langle \nabla h(\wv), \uv-\wv \rangle + \frac{L}{2} \| \uv-\wv \|^2  \qquad\forall \uv,\wv\in\R^d \, .
\label{eq:smooth}
\end{equation}
\end{definition}

\begin{definition}[$\mu$-Strong Convexity]
A function $h:\R^d\rightarrow\R$ is called \emph{$\mu$-strongly convex} w.r.t. a norm $\|.\|$, for $\mu\ge0$, if
\begin{equation}
h(\uv) \geq h(\wv) + \langle \nabla h(\wv), \uv-\wv \rangle + \frac{\mu}{2} \| \uv-\wv \|^2  \qquad\forall \uv,\wv\in\R^d \, .
\label{eq:strongconv}
\end{equation}
And analogously if the same holds for all subgradients, in the case of a general closed convex function $h$.
\end{definition}

\section{Convex Conjugates}\label{sec:conjugates}
We recall some basic properties of convex conjugates, which we use in the paper.

The convex conjugate of a function $f: \R^d\rightarrow \R$ is defined as 
\begin{equation}
f^*(\vv) := \sup_{\uv\in\R^d} \vv^\top \uv - f(\uv) \, .
\end{equation}
Some useful properties, see \citep[Section 3.3.2]{Boyd:2004uz}:
\begin{itemize}
\item Double conjugate: \hspace{2em}
$(f^*)^* = f$ if $f$ is closed and convex.
\item Value Scaling:  (for $\alpha>0$) \hspace{2em}
$
f(\vv) = \alpha g(\vv) 
\qquad\Rightarrow\qquad
f^*(\wv) = \alpha g^*(\wv/\alpha)    \, .
$
\item Argument Scaling:  (for $\alpha\ne0$) \hspace{2em}
$
f(\vv) = g(\alpha \vv) 
\qquad\Rightarrow\qquad
f^*(\wv) = g^*(\wv/\alpha) \, .
$
\item Conjugate of a separable sum: \hspace{2em}
$
f(\vv)=\sum_i \phi_i(v_i)
\qquad\Rightarrow\qquad
f^*( \wv ) = \sum_i \phi_i^* ( w_i ) \, .
$
\end{itemize}

\begin{replemma}{lem:dualLipschitz}[{Duality between Lipschitzness and L-Bounded Support. A generalization of~\citep[Corollary 13.3.3]{Rockafellar:1997ww}}]
Given a proper convex function $g$, it holds that
$g$ has $L$-bounded support w.r.t. the norm $\|.\|$ if and only if 
$g^*$ is $L$-Lipschitz w.r.t. the dual norm $\|.\|_*$.
\end{replemma}

\begin{lemma}[{Duality between Smoothness and Strong Convexity, \citep[Theorem 6]{Kakade:2009wh}}]
\label{lem:dualSmooth}
Given a closed convex function~$f$, it holds that
$f$ is $\mu$-strongly convex w.r.t. the norm $\|.\|$
if and only if
$f^*$ is $(1/{\mu})$-smooth w.r.t. the dual norm~$\|.\|_*$.
\end{lemma}
\begin{lemma}[Conjugates of Indicator Functions and Norms]~\vspace{-5mm}
\label{lem:NormConjugates}
\begin{enumerate}
\item[i)] The conjugate of the indicator function $\id_{\cC}$ of a set $\cC\subset \R^n$ (not necessarily convex) is the support function of the set $\cC$, that is
\[
\id_{\cC}^*(\xv) = \sup_{\sv\in\cC} \langle \sv,\xv\rangle
\]
\item[ii)] The conjugate of a norm is the indicator function of the unit ball of the dual norm.
\end{enumerate}
\end{lemma}
\begin{proof}
\citep[Examples 3.24 and 3.26]{Boyd:2004uz}
\end{proof}

\section{Primal-Dual Relationship}\label{app:duality}
The relation of the primal and dual problems \eqref{eq:A} and \eqref{eq:B} is a special case of the concept of Fenchel Duality. Using the combination with the linear map $A$ as in our case, the relationship is called \emph{Fenchel-Rockafellar Duality}, see, e.g., \citep[Theorem 4.4.2]{Borwein:2005ge} or \citep[Proposition 15.18]{Bauschke:2011ik}.

For completeness, we here illustrate this correspondence with a self-contained derivation of the duality.

Starting with the formulation \eqref{eq:A}, we introduce a helper variable $\vv\in \R^d$. The optimization problem~\eqref{eq:A} becomes:
\begin{equation}
\label{eq:constrainedprimal}
\min_{\alphav \in\R^{n}} \quad  f(\vv) + g( \alphav) \quad \text{such that} \ \vv =A\alphav \, .
\end{equation}
Introducing dual variables $\wv = [w_1, \dots, w_d]$,  the Lagrangian is given by:
$$ L(\alphav, \vv; \wv) := f(\vv) +   g(\alphav) + \wv^\top\left(A\alphav-\vv\right) \, .$$
The dual problem follows by taking the infimum with respect to $\alphav$ and $\vv$:
\begin{align}
\inf_{\alphav, \vv} L(\wv, \alphav, \vv) & =   \inf_{\vv} \left\{ f(\vv) - \wv^\top \vv \right\} + \inf_{\alphav} \left\{ g(\alphav) +  \wv^\top A\alphav\right\} \notag \\
& =  - \sup_{\vv} \left\{  \wv^\top \vv - f(\vv) \right\}- \sup_{\alphav} \left\{(-\wv^\top A)\alphav -  g(\alphav) \right\} \notag\\
& = - f^*(\wv) - g^*(-A^\top \wv) \label{eq:Lagrangian}\, .
\end{align}
We change signs and turn the  maximization of the dual problem \eqref{eq:Lagrangian} into a minimization and thus we arrive at the dual formulation $\eqref{eq:B}$ as claimed:
$$
    \min_{\wv \in \R^{d}} \quad \Big[ \ 
    \bP(\wv) := f^*(\wv) + g^*(-A^\top \wv) \ \Big] \, .
$$

\section{Proof of Lemma \ref{lemma:dualityGapAndDualSuboptimality}}
\label{sec:proofLemma1}
The proof is partially motivated by proofs in
\cite{ma2015adding,ShalevShwartz:2013wl}
but come with some crucial unique steps and tricks.

\begin{proof}[Proof of Lemma \ref{lemma:dualityGapAndDualSuboptimality}]
We have
\begin{align*}
\bD(\alphav)-\bD(\alphav^\star)
&=\bD(\alphav)-\min_{\delta \alphav} \bD(\alphav + \Delta\alphav)
\\
&=\max_{\delta \alphav}
\Big(
g(\alphav)
-g(\alphav+\Delta\alphav)
 +f(A \alphav)
 -f(A (\alphav+\Delta\alphav))
\Big)
\\
&=\max_{s \in [0,1]}
\Big(
g(\alphav)
-g(\alphav+s(\uv-\alphav))
 +f(A \alphav)
 - f(A (\alphav+s(\uv-\alphav)))
\Big)
\\
&\geq 
g(\alphav)
-g(\alphav+s(\uv-\alphav))
 +f(A \alphav)
- f(A (\alphav+s(\uv-\alphav)).
  \tagthis \label{adsfsafsafa}
\end{align*}

Here we have chosen $\Delta\alphav := s(\uv - \alphav)$ with $\uv$ defined as in  \eqref{asdfafsafsa} and some $s\in[0,1]$.  Note that for $\uv$ to be well defined, i.e., the subgradient in  \eqref{asdfafsafsa} not to be empty, we need the domain of $g^*$ to be the whole space. For $\mu>0$ this is given by strong convexity of $g$, while for $\mu=0$ this follows from the bounded support assumption on $g$. (For the duality of Lipschitzness and bounded support, see our Lemma \ref{lem:dualLipschitz}).

Now, we can use the fact, that function 
$f: \R^d \to \R$ has Lipschitz continuous gradient w.r.t. some norm $\|.\|_f$, with constant $1/\beta$,  
to obtain
\begin{align*}
\bD(\alphav)-\bD(\alphav^\star)
&\overset{\eqref{adsfsafsafa}}{\geq} 
g(\alphav)
-g(\alphav+s(\uv-\alphav))
+f(A \alphav)
 - f(A \alphav)
  -\ve{\nabla f(A \alphav) }{ A s(\uv-\alphav)}
  -
\frac{1}{2\beta} 
\|   A  s(\uv-\alphav) \|_f^2
\\
&=
g(\alphav)
-g(\alphav+s(\uv-\alphav))
 -\ve{\nabla f(A \alphav) }{ A s(\uv-\alphav)}
 - 
\frac{1}{2\beta} 
\|   A  s(\uv-\alphav) \|_f^2.
\tagthis \label{fasdfsadfasfdsa}
\end{align*}
Now, we will use a strong convexity property of a function $g$ w.r.t. $\|.\|_g$ to obtain
\begin{align*}
\bD(\alphav)-\bD(\alphav^\star)
&\overset{\eqref{fasdfsadfasfdsa}}{\geq}  
s
\underbrace{\Big(
g(\alphav)
- g(\uv)
-\ve{\nabla f(A \alphav) }{ A (\uv-\alphav)}
\Big)}_{\Lambda}
+
\frac{s^2}2
\Big(
\frac{\mu(1-s) }{s}
  \|\uv-\alphav\|_g^2
    - 
 \frac1\beta  
\|   A   (\uv-\alphav) \|_f^2
\Big).
\tagthis \label{asdfsafsafsa}
 \end{align*}

Now let us examine the relation of the equation above with duality gap.
Therefore we use the definition of the optimization problems \eqref{eq:A} and \eqref{eq:B}, and the definition of the convex conjugates to write the duality gap as:
\begin{align*}
G(\alphav )=\bP(\wv(\alphav))-(-\bD(\alphav))
&=g^*(-A^\top \wv(\alphav)) +g(\alphav)+ f^*(\wv(\alphav))
 +   f(A \alphav)
\\
&=g^*(-A^\top \wv(\alphav)) +g(\alphav)+ f^*(\nabla f(A\alphav))
 +   f(A \alphav) 
\\
&=g^*(-A^\top \wv(\alphav)) +g(\alphav)+ 
\ve{ \wv(\alphav)} {A \alphav}. 
\label{asfdsafasfsa}
\tagthis 
\end{align*}
where we have used the mapping $\wv(\alphav) = \nabla f(A\alphav)$. %

Now, let us analyze
the expression $\Lambda$ from
\eqref{asdfsafsafsa}. We have
\begin{align*}
\Lambda
&=
g(\alphav)
+\ve{\wv(  \alphav) }{ A  \alphav }
-g(\uv)
-\ve{\wv(  \alphav) }{ A \uv}.
\tagthis \label{adsfasfa}
 \end{align*}
Now, using the convex conjugate maximal property
and \eqref{asdfafsafsa}
we have
\begin{equation}
 g(\uv) = \ve{\uv}{-A^\top\wv(\alphav)} - g^*(-A^\top \wv(\alphav)).
 \label{adfsafsafsafa}
\end{equation}
Plugging
\eqref{adfsafsafsafa}, \eqref{asfdsafasfsa}
and \eqref{adsfasfa}
into \eqref{asdfsafsafsa}
gives us
\begin{align*}
\bD(\alphav)-\bD(\alphav^\star)
&\geq  
s
G(\alphav)
+
\frac{s^2}2
\Big(
\frac{\mu(1-s) }{s}
  \|\uv-\alphav\|_g^2
    - 
 \frac1\beta  
\|   A   (\uv-\alphav) \|_f^2
\Big).
 \end{align*}
and \eqref{eq:dualityGapAndDualSuboptimality} follows.
\end{proof}

\section{Proof of Theorem \ref{thm:convFastSCg}}

\begin{proof}
Let us upper bound the second term in
\eqref{eq:dualityGapAndDualSuboptimality} as given in the main Lemma~\ref{lemma:dualityGapAndDualSuboptimality}.
Recall the definition of the data complexity parameter $\sigma := \big(\max_{\alphav\ne0} \|A\alphav\|_f/\|\alphav\|_g\big)^2$.
We have
\begin{align*}
 \tfrac{\mu(1-s) }{s}
  \|\uv-\alphav\|_g^2
    - 
 \tfrac1\beta  
\|   A   (\uv-\alphav) \|_f^2
\geq  
\left[ \tfrac{\mu(1-s) }{s}
-\frac \sigma\beta\right]
  \|\uv-\alphav\|_g^2.
\end{align*}
Now, if we choose 
$
s = 
\frac{\mu}{ \frac \sigma\beta + \mu}   
$
then this term vanishes, and therefore
 \begin{align*}
 \frac{\mu}{\mu + \frac\sigma\beta }
\Exp[G(\vc{\alphav}{t})]
&\ \overset{\eqref{eq:dualityGapAndDualSuboptimality}}{\leq} \ 
\Exp[  \bD(\vc{\alphav}{t})-\bD(\alphav^\star)]
\ \overset{\eqref{eq:afcewwa}}{\leq} \ 
  (1-C)^t \ D.
\end{align*}
 After multiplying the equation above by
 $\frac{\frac\sigma\beta  +\mu}{\mu}$ 
 and requiring RHS to be $\leq \epsilon$
we will get 
\begin{align*}
 \frac{\frac\sigma\beta +\mu}{\mu} (1-C)^t D & \leq \epsilon,
\\
  t  &  \geq \frac{\log \frac{\mu \epsilon}{D(\frac\sigma\beta +\mu)}}{\log (1-C)}
\end{align*}
 and our claimed convergence bound \eqref{eqafsdfasdfas} follows.  
\end{proof}

\section{Proof of Theorem \ref{thm:convFastLipsch}}

\begin{proof}
From Lemma \ref{lemma:dualityGapAndDualSuboptimality} and the definition of $\sigma$
we have that
\begin{align}
s \label{Eq:afsdfawvfaw}
\Exp[
G(\vc{\alphav}{t})]
&\ \overset{\eqref{eq:dualityGapAndDualSuboptimality} }{\leq}\ 
\Exp\big[
\bD(\vc{\alphav}{t}) - \bD(\alphav^\star)
+
  \tfrac{s^2}{2\beta}  
\|   A   (\uv-\alphav) \|_f^2
\big]
\ \overset{\eqref{eq:afcewwa}}{\leq}\ 
(1-C)^t \ D
+
  \frac\sigma \beta \tfrac{s^2}2  
\Exp[\|     \uv-\alphav  \|_g^2] \ .
\end{align}

Now using Lemma \ref{lem:dualLipschitz}, because $g^*$ is $L$-Lipschitz w.r.t. the norm $\|.\|_*$, we have that $g$ is $L$-bounded w.r.t. the norm $\|.\|=\|.\|_g=\|.\|_f$ and hence for $\alphav\in \dom(g)$ we obtain $\|\alphav\|<L$. 
From the standard characterization of Lipschitzness as by bounded subgradient norm (see e.g. \citep[Lemma 2.6]{ShalevShwartz:2011dz} or other references) we have for any $\uv \in \partial g^*(\xv)$ that $\|\uv\|\leq L$, and hence for any $\uv, \alphav$ we must have 
$\|\uv - \alphav\|^2\leq 2L^2$ by the triangle inequality.

Therefore we conclude that
\begin{align}
\Exp[
G(\vc{\alphav}{t})] 
\overset{\eqref{Eq:afsdfawvfaw}}{\leq} 
 \frac1s (1-C)^t D
+
 s \frac \sigma\beta  L^2.
 \label{Eq:affeawvfrwaefc}
\end{align}
Now, let us choose 
$\bar s = \min\{1, \frac{\epsilon \beta}{2  \sigma L^2}\}$.
To have the RHS of 
\eqref{Eq:affeawvfrwaefc}
$\leq \frac\epsilon2$ it is enough to choose 
$$
t \geq T
=\frac{\log (\bar s \frac\epsilon{2D})}{ (1-C)}
$$
and the claimed convergence bound \eqref{eq:fsarwjvlawjs} follows.
\end{proof}

\section{Proof of Theorem \ref{thm:convBoundSup}}

\begin{proof}
Using Lemma \ref{lemma:dualityGapAndDualSuboptimality} and the bound $\|\uv - \alphav\|^2\leq 2L^2$ derived in the previous section,
we have that 
\begin{align*} 
 \Exp[ G(\vc{\alphav}{t})]
&\ \overset{\eqref{eq:dualityGapAndDualSuboptimality}}{\leq}\ 
\frac1s
\Exp[\bD(\vc{\alphav}{t})-\bD(\alphav^\star) ]
  +
\frac{s}2
  \frac\sigma\beta  L^2 \ \overset{\eqref{eq:sublinearAlgorithm}}{\leq}\  
\frac1s
\frac{C}{D(t)}
  +
\frac{s}2
 \frac\sigma\beta  L^2.
 \label{eq:afsewfwavfwav}\tagthis 
\end{align*}
Now, by choosing 
$s = \sqrt{ \frac{C}{D(t)} \frac{2\beta}{\sigma L^2} } \overset{\eqref{eq:asfdwafeawf}}{\in} [0,1]$
we obtain that
\begin{align} 
 \Exp[ G(\vc{\alphav}{t})]
&\overset{\eqref{eq:afsewfwavfwav}}{\leq}   
\sqrt{\frac{2 C \sigma L^2}{\beta  D(t)}}.
\end{align}
We see that the assumption \eqref{eq:asfdwafeawf} guarantees that the RHS of the above inequality becomes $\leq \epsilon$, as claimed.
\end{proof}

We note that in the special case of constrained optimization, motivated by the Frank-Wolfe algorithm, 
\citep[Theorem 2]{LacosteJulien:2015wj} has shown another algorithm-independent bound for the convergence in duality gap, also requiring order of $1/\epsilon^2$ steps to reach accuracy $\epsilon$. On the other hand, for the algorithm-specific result by \citep[Prop 4.2]{Bach:2015bz} for the case of an extended Frank-Wolfe algorithm, $1/\epsilon$ steps are shown to be sufficient.

\section{Proof of Lemma \ref{lem:dualLipschitz}}
Lemma \ref{lem:dualLipschitz} generalizes the result of \citep[Corollary 13.3.3]{Rockafellar:1997ww} from the $L_2$-norm to a general norm $\|.\|$. 
It can be equivalently formulated as follows:

\begin{lemma}
Let $f:\R^n \rightarrow \R$ be a proper convex function. In order that dom($f^*$) be $L$-bounded w.r.t $\|.\|_*$ (a real number $L\geq 0$), it is necessary and sufficient that $f$ be finite everywhere and that 
\[|f(\zv)-f(\xv)| \leq L \|\zv-\xv\|\;\;\;\forall \xv,\zv \]
\end{lemma}
\begin{proof}
We follow the line of proof in \citep{Rockafellar:1997ww}.

It can be assumed that $f$ is closed, as $f$ and its closure $cl(f)$ have the same conjugate, and the Lipschitz condition is satisfied by $f$ if and only if it is satisfied by $cl(f)$. 
Note that from  \citep[Theorem 13.3]{Rockafellar:1997ww} we know 
\begin{equation} \label{eq:Deff0}
f0^+(\yv):=\sup_{\xv\in \dom(f)}f(\xv+\yv)- f(\xv)
\end{equation}
is the support function of $\dom(f^*)$. Further, from \citep[Theorem 10.5]{Rockafellar:1997ww} we know: $\dom(f^*)$ is bounded if and only if  $f0^+$ is finite everywhere.
In more detail, the Lipschitz condition on $f$, choosing $\zv=\xv+\yv$, is equivalent to having
\[f(\xv+\yv)- f(\xv)\leq L \|\yv\|,\;\;\; \forall \xv,\yv\]
and by definition \eqref{eq:Deff0} of $f0^+$ this is equivalent to 
\[f0^+(\yv)\leq L \|\yv\|,\;\;\; \forall \yv\]
Note that $g(\yv):=L\|\yv\|$ is a finite positively homogenous convex function, we know by \citep[Corollary 13.2.2]{Rockafellar:1997ww} that it is the support function of a non-empty bounded convex set. Call this set $\bS$.
We will later show that $\bS=L \bB$ where $\bB$ is the $\|.\|_*$-norm ball.
 Hence $f0^+(\yv)\leq g$ means $cl(\dom f^*)\subset \bS$, as $f0^+$ is the support function of $\dom(f^*)$. And this shows that the Lipschitz condition holds for some $L$ if and only if $\xv^\star \in \bS\;\;\forall \xv^\star\in \dom(f^*)$ and hence $\|\xv^\star\|_*\leq L$ and $f^*$ is $L$-bounded for every $\xv^\star \in \dom(f^*)$.
 \newline
 It remains to show that $g=L\|\yv\|$ is the support function of $\bS=L \bB$:
 The support function of a set $\cC$ is defined as 
\[\delta_\cC (\yv):= \sup_{\xv\in \cC}\langle \yv,\xv\rangle.\] 
By the definition of the dual norm we find
\[L \|\yv\| = L \max_{\xv:\|\xv\|_*\leq1} \langle \yv,\xv\rangle  =  \max_{\xv:\|\xv\|_*\leq L} \langle \yv,\xv\rangle \]
and hence $g$ is the support function of the set $\bS:=\{\xv:\|\xv\|_*\leq L\}=L \bB$.
\end{proof}

\section{Duality Gaps}

\subsection{Duality Gap for Lasso (Using Lipschitzing)}
\label{sec:DualityGapLasso}

Consider the Lasso problem given in \eqref{eq:LassoObj}, which we directly map to our primal optimization problem \eqref{eq:A}. 
We use the $L_1$-norm ball of radius $B$ to apply the Lipschitzing trick, i.e. $\bB:=\left\{\xv:\|\xv\|_1\leq B\right\}$ and modify the $L_1$ norm term $g(\alphav):=\lambda \|\alphav\|_1$ as suggested in \eqref{eq:Lip}. The convex conjugate hence becomes:
\begin{equation}
\bar g^*(\uv) =B\left[\|\uv\|_\infty - \lambda\right]_+\label{eq:dualL1}
\end{equation}
and using the optimality condition  \eqref{eq:opt_f} the gap follows immediately as
\begin{align*}
\bar G(\alphav) &=f(A\alphav) + f^*(\wv(\alphav))+ B \left[\| A^\top\wv(\alphav)\|_\infty-\lambda\right]_+ +\lambda \|\alphav\|_1\\
&=\langle \nabla f(A \alphav), A \alphav\rangle+ B \left[\| A^\top\wv(\alphav)\|_\infty-\lambda\right]_+ +\lambda \|\alphav\|_1
\end{align*}
 where we used the first-order optimality $f^*(\wv(\alphav)) = \langle \wv(\alphav), A \alphav\rangle - f(A\alphav)$ by definition of $\wv(\alphav)$.
\begin{proof}
From \eqref{eq:barg*} we know
\begin{equation}\bar g^*(\uv) = 
\begin{cases}
0 & \|\uv\|_\infty \leq \lambda\\
\displaystyle\max_{ \alphav: \| \alphav\|_1 \leq B } \uv^\top \alphav - \lambda\|\alphav\|_1& \text{else}
\end{cases}.
\end{equation}
It remains to consider the non zero part, where $ \|\uv\|_\infty > \lambda$. 
Lets consider the optimization 
\[\max_{ \alphav: \| \alphav\|_1 \leq B } \uv^\top \alphav - \lambda\|\alphav\|_1.\]
Assume the largest element of $\uv$ is at index $i$ and we know $[\uv]_i>\lambda$, hence the maximizing $\alphav$ in $\{ \alphav: \| \alphav\|_1 \leq B \}$ is the vector with a single entry of value $B$ at position $i$, and \eqref{eq:dualL1} follows.
\end{proof}

\subsection{Duality Gap for Elastic Net Regularized Problems}
\label{sec:DualityGapElastic}

Consider the Elastic Net regularized problem formulation \eqref{eq:elasticNet}. The conjugate dual of the Elastic Net regularizer is given in Lemma \ref{lem:elasticnetconjugate}. We first map \eqref{eq:elasticNet} to \eqref{eq:A}, as suggested in row 2 of Table \ref{tbl:MLProblems}. Then its dual counterpart \eqref{eq:B} is given by:
\[\bP(\wv) = \ell^*(\wv) + \tfrac{1}{2\eta \lambda}\sum_{i=1}^n \left(\left [\left|A_{:i}^\top\wv\right|-(1-\eta)\lambda\right ]_+\right)^2.\]
In light of the optimality condition \eqref{eq:opt_f} we now use the definition of the primal-dual correspondence $\wv(\alphav) = \nabla \ell(A \alphav)$, and hence obtain
\begin{align*}
G(\alphav) &=(A\alphav)^\top\nabla \ell(A\alphav)+\tfrac{1}{2\eta \lambda}\sum_{i =1}^n \left( \left[\left|A_{:i}^\top\wv\right|-(1-\eta)\lambda\right]_+\right)^2 +\lambda \left(\frac{\eta}{2} \|\alphav\|_2^2 + (1-\eta)\|\alphav\|_1\right).
\end{align*}
Note that when alternatively mapping \eqref{eq:elasticNet} to \eqref{eq:B} (as suggested in row 3 of Table \ref{tbl:MLProblems}) instead of \eqref{eq:A}, the duality gap function will be similarly structured, but in that case the variable mapping $\wv(\alphav)$ will be defined using the gradient of the Elastic Net penalty, instead of $\ell$.

\section{Convergence Results for Coordinate Descent}
\label{sec:CDproofs}
For the proof of  Theorem \ref{thm:SDCAthm5} and \ref{thm:SDCAthm2} and we will need the following Lemma, which we will prove below in Section \ref{sec:lemBasic}. This lemma allows us to lower bound the expected per-step improvement in terms of $\bD$ for coordinate descent algorithms on $\bD$. 

\begin{lemma}\label{lemma:basic}%
Consider problem formulation \eqref{eq:Psep} and \eqref{eq:Dsep}. Let $f$ be $1/\beta$-smooth w.r.t. a norm $\|.\|$. Further, let $g_i$ be $\mu$-strongly convex with convexity parameter $\mu \geq 0$ $\forall i\in[n]$. For the case $\mu=0$ we need the additional assumption of $g_i$ having bounded support.
Then for any iteration $t$  and any $s\in [0,1]$, it holds that
\begin{align}
&\Exp[ \bD(\vc{\alphav}{t-1}) - \bD(\vc{\alphav}{t})] \geq \frac{s}{n} G(\alphav^{t-1}) -\frac{s^2 F^{(t)}}{2},
\end{align}
where
\begin{align}
 \vc{F}{t}&:= \frac{1}{n}\sum_{i=1}^n\left(  \frac{1}{ \beta}\|A_{:i}\|^2-\frac{(1-s)\mu}{ s}\right)  \Exp\left[\big(\vc{u_i}{t-1}-\vc{\alpha_i}{t-1}\big)^2\right],\label{eq:F}
\end{align} 
and $\vc{u_i}{t-1} \in \partial g_i^*(-A_{:i}^\top\wv(\vc{\alphav}{t-1}))$. The expectations are with respect to the random choice of coordinate $i$ in step $t$.
\end{lemma}

\subsection{Proof of Theorem \ref{thm:SDCAthm5}}
\label{sec:proofThm5}

For the following, we again assume that the columns of the data matrix $A$ are scaled such that $\|A_{:i}\| \leq R$ for all $i\in[n]$ and $\|A_{j:}\| \leq P$ for all $ j\in[d]$ and some norm $\|.\|$.

\begin{reptheorem}{thm:SDCAthm5}
Consider Algorithm \ref{alg:SDCA} applied to \eqref{eq:Dsep}. Assume $f$ is a $1/\beta$-smooth function w.r.t. the norm $\|.\|$. Then, if $g_i$ is $\mu$-strongly convex for all~$i$, it suffices to have a total number of iterations of 
\begin{align*}
T&\geq \left(n+\tfrac{n R^2}{\mu \beta}\right) \log\left( \left[n+\tfrac{n R^2}{\mu \beta}\right] \tfrac{\epsilon_D^{(0)}}{\epsilon}\right)
\end{align*}
to get $\Exp[G(\vc{\alphav}{T})] \leq \epsilon$.
Moreover, to obtain an expected duality gap of $\Exp[G(\bar \alphav)] \leq \epsilon$ it suffices to have $T>T_0$ with 
\begin{align*}
T_0&\geq \left(n+\tfrac{n R^2}{\mu \beta}\right) \log\left( \left[n+\tfrac{n R^2}{\mu \beta}\right] \tfrac{\epsilon_D^{(0)}}{(T-T_0)\epsilon}\right)
\end{align*}
where $ \epsilon_D^{(0)}$ is the initial suboptimality in $\bD(\alphav)$.
\end{reptheorem}

\begin{proof}
The proof of this theorem is motivated by proofs in \citep{ShalevShwartz:2013wl} but generalizing their result in \citep[Theorem 5]{ShalevShwartz:2013wl}.
To prove Theorem \ref{thm:SDCAthm5} we apply Lemma \ref{lemma:basic} with $s=\frac{\mu \beta}{R^2+\mu \beta}\in[0,1]$. This choice of $s$ implies $\vc{F}{t}\leq0\;\;\forall t$ as defined in \eqref{eq:F}. Hence, 
\[
\Exp\big[ \bD(\vc{\alphav}{t-1}) - \bD(\vc{\alphav}{t}) \big]
\geq \frac s n G(\vc \alphav {t-1}) = \frac{s}{n} \bP(\vc{\wv}{t-1})+\bD(\vc{\alphav}{t-1}) .
\]
Here expectations are over the choice of coordinate $i$ in step $t$, conditioned on the past state at $t-1$.
Let $\epsilon_D^{(t)}$ denote the suboptimality $\epsilon_D^{(t)} := D(\vc \alphav t)-D(\alphav^\star)$. As
$\epsilon_D^{(t-1)} \leq  P(\vc \wv{t-1})+D(\vc \alphav{t-1})$ and $D(\vc \alphav {t-1})-D(\vc \alphav t)= \vc{ \epsilon_D}{t-1}-\vc{\epsilon_D}{t}$, we obtain
\begin{eqnarray*}
\Exp[\vc{\epsilon_D}{t}]
\leq \left(1-\frac{s}{n}\right) \vc{\epsilon_D}{t-1}
\leq \left(1-\frac{s}{n}\right)^t  \vc{\epsilon_D}{0}
\leq \exp(-st/n) \vc{\epsilon_D}{0}
\leq \exp\left(-\frac{\mu \beta t}{n(R^2+ \mu\beta)}\right) \vc{\epsilon_D}{0}.
\end{eqnarray*}
To finally upper bound the expected suboptimality as $\Exp[\vc{\epsilon_D}{t}]\leq \epsilon_D$ we need
\[t \geq \left(\frac{nR^2}{\mu \beta}+n\right)\log\left(\frac{ \vc{\epsilon_D}{0}}{\epsilon_D}\right).\] 
And towards obtaining small duality gap as well, we observe that at all times
 \begin{equation}\label{eq:Dpd}
\Exp[G(\vc{ \alphav}{t-1})]
\leq \frac{n}{s}\Exp[\vc{\epsilon_D}{t-1}-\vc{\epsilon_D}{t}]
\leq \frac{n}{s} \Exp[\vc{\epsilon_D}{t-1}] \ .
\end{equation}
Hence with $\epsilon\geq\frac{n}{s}\Exp[\vc{\epsilon_D}{t}]$ we must have a duality gap smaller than  $\epsilon$. Therefore we require %
\[
t \geq \left(\frac{nR^2}{\mu \beta}+n\right)\log\left(\left(n+\frac{n R^2}{\mu \beta}\right)\frac{ \epsilon_D^{0}}{\epsilon}\right).
\]
This proves the first part of Theorem \ref{thm:SDCAthm5} and the second part follows immediately if we sum \eqref{eq:Dpd} over $t=T_0,...,T-1$.
\end{proof}

\begin{remark}
From Theorem \ref{thm:SDCAthm5} it follows that for $T=2T_0$ and $T_0\geq n+\frac{n R^2}{\mu \beta}$ we need 
\[T\geq 2 \left(n+\frac{n R^2}{\mu \beta} \right)\log\left(\frac{\epsilon_D^{(0)}}{\epsilon}\right)\]
\end{remark}

\paragraph{Recovering SDCA as a Special Case.}

\begin{corollary} \label{cor:Thm5SDCA}
We recover Theorem 5 of \cite{ShalevShwartz:2013wl} as a special case of our Theorem \ref{thm:SDCAthm5}. We consider their optimization objectives, namely
\begin{align}
    \label{eq:primalSDCA}
    \bP(\wv)& := \frac{1}{n}\sum_{i=1}^n \Phi_i( \xv_i^\top \wv) + \frac{\lambda}{2}\|\wv\|^2   \\
    \label{eq:dualSDCA}
    \bD(\alphav) &:=-\left[ \frac{1}{n}\sum_{i=1}^n -\Phi^*_i(-\alpha_i) - \frac{\lambda}{2}\left\|\frac{1}{\lambda n}\sum_{i=1}^n \alpha_i \xv_i\right\|^2\right] .
\end{align}
where $\xv_i$ are the columns of the data matrix $X$.  We assume that   $\Phi_i^*(\alphav)$ is $\gamma$-strongly convex for $i\in[n]$. We scale the columns of $X$ such that $ \|\xv_i\|\leq 1$. Hence, we find that 
\begin{align*}
T_0&\geq \left(n+\frac{1}{\gamma \lambda}\right) \log\left(\left[n+\frac{1}{\gamma \lambda}\right]\frac{1}{\epsilon (T-T_0)}\right)
\end{align*}
iterations are sufficient to obtain a duality gap of $\Exp[\bP( \bar\wv )-(-\bD(\bar \alphav))]\leq  \epsilon$.
\end{corollary}
\begin{proof}
We consider \eqref{eq:primalSDCA} and \eqref{eq:dualSDCA} as a special case of the separable problems \eqref{eq:Dsep} and \eqref{eq:Psep}.
We set $g_i(\alpha):=\frac{1}{n}\Phi_i^*(-\alpha)$ and  $f^*(\wv):=\frac{\lambda}{2} \|\wv\|^2$. In this case $\mu:=\frac{1}{n}\gamma$ and $\beta := \lambda$. Defining $A:=\frac{1}{n}X$ and using the assumption  $ \|\xv_i\|\leq 1$  we have $R:=\frac{1}{n}$ and using $\Phi_i(0)\leq1$ we have $\epsilon_D^{(0)}\leq 1$and applying Theorem \ref{thm:SDCAthm5} to this setting concludes the proof.
\end{proof}

\subsection{Proof of Theorem \ref{thm:SDCAthm2}}
\label{sec:proofThm2}
This theorem generalizes the results of \citep[Theorem 2]{ShalevShwartz:2013wl}. %

\begin{reptheorem}{thm:SDCAthm2}%
Consider Algorithm \ref{alg:SDCA} applied to \eqref{eq:Dsep}. Assume $f$ is a $1/\beta$-smooth function w.r.t. the norm $\|.\|$. Then, if $g_i^*$ is $L$-Lipschitz for all $i$,  it suffices to have a total number of iterations of 
\begin{align*}
T&\geq \max\left\{0,n\log\frac{ \epsilon_D^{(0)} \beta}{ 2L^2R^{2} n}\right\}+  n +\frac{20 n^2 L^2  R^2}{\beta \epsilon }
\end{align*}
to get $\Exp[G(\bar \alphav)] \leq \epsilon$. Moreover, when $t\geq T_0$ with
\[T_0= \max\left\{0,n\log\frac{ \epsilon_D^{(0)} \beta}{ 2L^2R^{2} n}\right\}+\frac{16 n^2L^2R^2}{ \beta \epsilon}\]
 we have the suboptimality bound of $\Exp[\bD(\alphav^{(t)})-\bD(\alphav^\star)]\leq \epsilon/2$, where $ \epsilon_D^{(0)}$ is the initial suboptimality.
\end{reptheorem}

We rely on the following Lemma:

\begin{lemma}\label{lemma:F}
Suppose that for all $i$, $g_i^*$ is $L$-Lipschitz. Let $ F^{(t)}$ be as defined in Lemma \ref{lemma:basic} (with $\mu=0$) and assume $\|A_{:i}\|\leq R\;\;\; \forall i$. Then,  $F^{(t)}\leq \frac{4L^2}{\beta}R^2\:\:\forall t$.
\end{lemma}
\begin{proof}
We apply the duality of Lipschitzness and bounded support (Lemma \ref{lem:dualLipschitz}) to the case of univariate functions $g_i$.
The assumption of $g_i^*$ being $L$-Lipschitz therefore gives that $g_i$ has $L$-bounded support, and so $|\alpha_i|\leq L$ for $\alpha_i$ as defined in~\eqref{eq:F}.

Furthermore for $u_i$, by the equivalence of Lipschitzness and bounded subgradient (see e.g. \citep[Lemma 2.6]{ShalevShwartz:2011dz}) we have $|u_i|\leq L$ and thus, $|\alpha_i-u_i|^2\leq 4 L^2$.
Together with $\|A_{:i}\|\leq R$ the claimed bound on $F^{(t)}$ follows.
\end{proof}

\begin{remark} In the case of Lasso we have $g_i(\alpha_i)= \lambda |\alpha_i|$. Using the Lipschitzing trick, we replace $g_i(\cdot)$ by
\begin{eqnarray*}
\bar{g}_i(\alpha) = \begin{cases}
          \lambda |\alpha| & : \alpha \in [-B,B]  \\
            +\infty & : \text{otherwise,}
        \end{cases}
\end{eqnarray*}
which has $B$-bounded support. Hence, its conjugate
\[
    \bar{g}^*(x) = 
    \begin{cases}
            0 & : x \in [-\lambda,\lambda]  \\
            B( |x| - \lambda) & : \text{otherwise,}
        \end{cases}
\]
is $B$-Lipschitz and Lemma \ref{lem:dualLipschitz} and Lemma \ref{lemma:F} apply with $L:=B$.
\end{remark}

\begin{proof}[Proof of Theorem \ref{thm:SDCAthm2}]
Now, to prove Theorem \ref{thm:SDCAthm2}, let $F=\max_t \vc{F}{t}$ and recall that by Lemma \ref{lemma:F} we can upper bound $F$ by $\frac{4L^2}{\beta}R^2$.

Furthermore, our main Lemma \ref{lemma:basic} on the improvement per step tells us that
\begin{align}
\label{eq:lemma:basic2}
&\Exp\big[ \bD(\vc{\alphav}{t-1})- \bD(\vc{\alphav}{t})\big] \geq \frac{s}{n} G(\vc \alphav {t-1}) -\frac{s^2}{2} F, 
\end{align}

With $\vc {\epsilon_D}t = D(\vc \alphav t)-D(\alphav^\star)\leq G(\vc \alphav t)$ and $D(\vc \alphav t)-D(\vc \alphav {t-1})= \epsilon_D^{(t-1)}-\vc {\epsilon_D}t$, this  implies
\begin{eqnarray*}
\Exp\big[\vc {\epsilon_D}{t-1}-\vc {\epsilon_D}{t}\big]&\geq&\frac{s}{n} \vc {\epsilon_D}{t-1} -  \frac{ s^2  F}{2}\\
\vc {\epsilon_D}{t-1}-\Exp[\vc {\epsilon_D}{t}]&\geq&\frac{s}{n} \vc {\epsilon_D}{t-1} -\frac{ s^2  F}{2}\\
\Exp[\vc {\epsilon_D}{t}]&\leq&\left(1-\frac{s}{n}\right) \vc {\epsilon_D}{t-1} +\frac{ s^2  F}{2}\\
\end{eqnarray*}
Expectations again only being over the choice of $i$ in steps $t$. 
We next show that this inequality can be used to bound the suboptimality as
\begin{equation}\label{eq:subOpt}
\Exp[\vc {\epsilon_D}{t}]\leq\frac{2 F n^2}{2n+t-t_0}
\end{equation}
for $t\geq t_0=\max\left\{0,n \log\left(\frac{2\epsilon_D^{(0)}}{F n}\right)\right\}$.
Indeed, let us choose $s:=1$. Then at $t=t_0$, we have
\begin{eqnarray*}
\Exp[\epsilon_D^{(t)}]&\leq& \left(1-\frac{1}{n}\right)^t \epsilon_D^{(0)} +\sum_{i=0}^{t-1} \left(1-\frac{1}{n}\right)^i \frac{ F}{2}\\
&\leq& \left(1-\frac{1}{n}\right)^t \epsilon_D^{(0)} +\frac{1}{1-(1-1/n)} \frac{ F}{2}\\
&\leq& e^{-t/n} \epsilon_D^{(0)} + \frac{n  F }{2}\\
&\leq&   F n
\end{eqnarray*}
For $t>t_0$ we use an inductive argument. Suppose the claim holds for $t-1$, therefore
\begin{eqnarray*}
\Exp[\epsilon_D^{(t)}]&\leq& \left(1-\frac{s}{n}\right) \Exp[\epsilon_D^{(t-1)}] +\frac{s^2 F}{2}\\
&\leq& \left(1-\frac{s}{n}\right)\frac{2  F n^2}{2n+(t-1)-t_0} +\frac{s^2 F}{2}
\end{eqnarray*}
choosing $s=\frac{2n}{2n+t-1-t_0}\in[0,1]$ yields
\begin{eqnarray*}
\Exp[\epsilon_D^{(t)}]&\leq& \left(1-\frac{2}{2n+t-1-t_0}\right)\frac{2  F n^2}{2n+(t-1)-t_0} +\left(\frac{2n}{2n+t-1-t_0}\right)^2\frac{ F}{2}\\
&=& \left(1-\frac{2}{2n+t-1-t_0}\right)\frac{2  F n^2}{2n+(t-1)-t_0} +\left(\frac{1}{2n+t-1-t_0}\right)\frac{2 F n^2}{2n+t-1-t_0}\\
&=& \left(1-\frac{1}{2n+t-1-t_0}\right)\frac{2  F n^2}{2n+(t-1)-t_0} \\
&=& \frac{2  F n^2}{(2n+t-1-t_0)} \frac{2n+t-2-t_0}{2n+t-1-t_0}\\
&\leq& \frac{2  F n^2}{(2n+t-t_0)}
\end{eqnarray*}
This proves the bound \eqref{eq:subOpt} on the suboptimality.
To get a result on the duality gap we sum \eqref{eq:lemma:basic2} over the interval $t=T_0+1,...,T$ and obtain
\[\Exp[  \bD(\vc{\alphav}{T_0})-\bD(\vc{\alphav}{T})] \geq \frac{s}{n} \Exp\left[\sum_{t=T_0+1}^T P(\wv^{(t-1)})+D(\alphav^{(t-1)})\right] - (T-T_0)\frac{s^2}{2}  F,\]
and rearranging terms we get
\[ \Exp\left[\frac{1}{T-T_0}\sum_{t=T_0+1}^T P(\wv^{(t-1)})+D(\alphav^{(t-1)})\right] \leq \frac{n}{s (T-T_0)}  \Exp[ \bD(\vc{\alphav}{T_0})- \bD(\vc{\alphav}{T}) ]+\frac{s n}{2}  F,\]
Now if we choose $\bar \wv, \bar \alphav$ to be  the average vectors over $t \in \{T_0+1,T\}$, then the above implies 
\[\Exp[G(\bar \alphav)] =  \Exp\left[ P(\bar \wv)+D(\bar \alphav)\right] \leq \frac{n}{s (T-T_0)}  \Exp[  \bD(\vc{\alphav}{T_0})-\bD(\vc{\alphav}{T}) ]+\frac{s n}{2}  F,\]
If $T\geq n+T_0$ and $T_0\geq t_0$, we can set $s=n/(T-T_0)$ and combining this with \eqref{eq:subOpt} we obtain
\begin{eqnarray*} 
\Exp[G(\bar \alphav)] =& \leq&   \Exp[ \bD(\vc{\alphav}{T_0})- \bD(\vc{\alphav}{T}) ]+\frac{ F n^2}{2 (T-T_0)} \\
& \leq&   \Exp[   \bD(\vc{\alphav}{T_0})-\bD(\alphav^\star)]+\frac{ F n^2}{2 (T-T_0)} \\
& \leq&  \frac{2  F n^2}{2n+t-t_0}+\frac{  F n^2}{2 (T-T_0)} 
\end{eqnarray*}
A sufficient condition to upper bound the duality gap by $\epsilon$ is that $T_0\geq \frac{4 F n^2}{ \epsilon}-2n+t_0$ and $T\geq T_0 + \frac{ F n^2}{\epsilon }$ which also implies $\Exp[  \bD(\vc{\alphav}{T_0})-\bD(\alphav^\star) ]\leq \epsilon/2$. Since we further require $T_0\geq t_0$ and $T-T_0\geq n$, the overall number of required iterations has to satisfy
\[
T_0\geq \max\Big\{t_0,\frac{4 F n^2}{ \epsilon}-2n+t_0\Big\} \;\;\text{and}\;\; T-T_0\geq\max\Big\{n,\frac{ F n^2}{\epsilon}\Big\}
\]
Using Lemma \ref{lemma:F} we can bound the total number of required iterations to reach a duality gap of $\epsilon$ by
\begin{eqnarray*}
T &\geq &T_0 + n +\frac{ n^2 }{\epsilon }\frac{4 L^2  R^2}{\beta}\\
 &\geq &t_0 + \frac{4 n^2}{ \epsilon}\frac{4 L^2  R^2}{\beta}+  n +\frac{ n^2 }{\epsilon }\frac{4 L^2  R^2}{\beta}\\
&\geq & \max\Big\{0,n\log\Big( \frac{\epsilon_D^{(0)}\beta}{2 n R^2 L^2}\Big)\Big\} +  n +\frac{20 n^2 L^2  R^2}{\beta \epsilon }
\end{eqnarray*}
which concludes the proof of Theorem \ref{thm:SDCAthm2}.
\end{proof}

\paragraph{Recovering SDCA as a Special Case.}

\begin{corollary} \label{cor:Thm2SDCA}
We recover Theorem 2 of \cite{ShalevShwartz:2013wl} as a special case of our Theorem \ref{thm:SDCAthm2}. We consider the optimization objectives in \eqref{eq:primalSDCA} and \eqref{eq:dualSDCA}. We assume that   $\Phi_i(\alphav)$ is $M$-Lipschitz for $i\in[n]$ and $\Phi_i(0)\leq 1$. We scale the columns of $X$ such that $ \|\xv_i\|\leq 1$. Hence, we find that 
\begin{align*}
T&\geq T_0 + n +\frac{4 M^2 }{\lambda \epsilon}\\
&\geq \max\{0,n\log(0.5 n \lambda M^{-2} )\}+  n +\frac{20  M^2 }{\lambda \epsilon }
\end{align*}
iterations are sufficient to obtain a duality gap of $\Exp[G(\bar \alphav)]\leq  \epsilon$.
\end{corollary}
\begin{proof}
We consider \eqref{eq:primalSDCA} and \eqref{eq:dualSDCA} as a special case \eqref{eq:Psep} and \eqref{eq:Dsep}.
We set $g_i(\alpha):=\frac{1}{n}\Phi_i^*(-\alpha)$ and  $f^*(\wv):=\frac{\lambda}{2} \|\wv\|^2$. Hence, $g_i^*(\wv^\top \av_i)=\tfrac{1}{n} \Phi_i(-n \wv^\top \av_i) $ is $M$-Lipschitz and we set $L:=M$ and $\beta := \lambda$.
We further use Lemma \citep[Lemma 20]{ShalevShwartz:2013wl} with $\Phi_i(0)\leq 1$ to bound the primal suboptimality as  $\epsilon_D^{(0)}\leq 1$.  Finally, by the assumption $\|\xv_i\|\leq 1$ and the definition $A:=\frac{1}{n}X$  we have $R:=\frac{1}{n}$ and applying Theorem \ref{thm:SDCAthm2} to this setting concludes the proof.
\end{proof}

\subsection{Proof of Lemma \ref{lemma:basic}}
\label{sec:lemBasic}
\begin{proof}
The proof of Lemma \ref{lemma:basic} is motivated by the proof of  \citep[Lemma 19]{ShalevShwartz:2013wl} but we adapt it to apply to a much more general setting. First note that the one step improvement in the dual objective can be written as
\begin{eqnarray*}
 \bD(\vc{\alphav}{t-1}) - \bD(\vc{\alphav}{t}) = \sum_{i = 1}^{n} g_i(\vc{\alpha_i}{t-1}) + f(A\vc{\alphav}{t-1}) - \left[  \sum_{i = 1}^{n} g_i(\vc{\alpha_i}{t}) + f(A\vc{\alphav}{t})   \right]
\end{eqnarray*}
Note that in a single step of SDCA $\vc{\alphav}{t-1}\rightarrow \vc{\alphav}{t}$ only one dual coordinate is changed. Without loss of generality we assume this coordinate to be $i$. Writing $\vv(\alphav) = A\alphav$, we find
\begin{eqnarray*}
 \bD(\vc{\alphav}{t-1}) - \bD(\vc{\alphav}{t})& =& \underbrace{ \left[ g_i(\vc{\alpha_i}{t-1}) + f(\vv(\vc{\alphav}{t-1}))   \right]}_{(\Gamma)}-\underbrace{\left[  g_i(\vc{\alpha_i}{t}) +f(\vv(\vc{\alphav}{t}))\right]}_{(\Lambda)} .
\end{eqnarray*}
In the following, let us denote the columns of the matrix $A$ by $\av_i$ for $i\in [n]$ for reasons of readability. Then, by definition of the update we have for all $s\in[0,1]$:
\begin{eqnarray*}
(\Lambda)& =&    g_i(\vc{\alpha_i}{t}) + f(\vv(\vc{\alphav}{t}))\\
& =& \min_{\Delta \alpha_i}\left[ g_i(\vc{\alpha_i}{t-1}+\Delta\alpha_i) +  f(\vv(\vc{\alphav}{t-1}+\ev_i \Delta \alpha_i))\right]\\
& =& \min_{\Delta \alpha_i}\left[ g_i(\vc{\alpha_i}{t-1}+\Delta\alpha_i) +  f\left(\vv(\vc{\alphav}{t-1})+\av_i \Delta\alpha_i\right)\right]\\
&\leq&\left[ g_i(\vc{\alpha_i}{t-1}+s(\vc{u_i}{t-1}-\vc{\alpha_i}{t-1})) + f\left(\vv(\vc{\alphav}{t-1})+\av_i s(\vc{u_i}{t-1}-\vc{\alpha_i}{t-1})\right)\right]\\
\end{eqnarray*}
where we chose $\Delta \alpha_i = s(\vc{u_i}{t-1}-\vc{\alpha_i}{t-1})$ with $\vc{-u_i}{t-1}\in \partial g_i^*(\av_i^\top \vc{\wv}{t-1}) $ for $s\in[0,1]$. As in Section \ref{sec:proofLemma1}, in order for $\uv$ to be well-defined, we again need the assumption on $g$ to be strongly convex ($\mu>0$), or to have bounded support.
\newline
Using $\mu$-strong convexity of $g_i$, namely
\begin{eqnarray*} g_i(\alpha_i^{(t-1)}+s(u_i^{(t-1)}-\alpha_i^{(t-1)}))& =&g_i(s u_i^{(t-1)}+(1-s)\alpha_i^{(t-1)})\\&\leq& s g_i(u_i^{(t-1)})+(1-s)g_i(\alpha_i^{(t-1)}) - \tfrac{\mu}{2}s(1-s)(u_i^{(t-1)}-\alpha_i^{(t-1)})^2,
\end{eqnarray*}
and $\frac{1}{\beta}$-smoothness of $f$
\[f(\vv(\alphav^{(t-1)})+s(u_i^{(t-1)}-\alpha_i^{(t-1)})\av_i) \leq f(\vv(\alphav^{(t-1)}))+\big\langle\nabla f(\vv(\alphav^{(t-1)})),s(u_i^{(t-1)}-\alpha_i^{(t-1)})\av_i\big\rangle +\tfrac{1}{2 \beta}\big\|s(u_i^{(t-1)}-\alpha_i^{(t-1)})\av_i\big\|^2\]
we find that
\begin{eqnarray*}
(\Lambda)&\leq& \left[ s g_i(u_i^{(t-1)})+(1-s)g_i(\alpha_i^{(t-1)}) - \tfrac{\mu}{2}s(1-s)(u_i^{(t-1)}-\alpha_i^{(t-1)})^2 \right] \\&& + \left[     f(\vv(\alphav^{(t-1)}))+\langle\nabla f(\vv(\alphav^{(t-1)})),s(u_i^{(t-1)}-\alpha_i^{(t-1)})\av_i\rangle +\tfrac{1}{2 \beta}\big\|s(u_i^{(t-1)}-\alpha_i^{(t-1)})\av_i\big\|^2 \right].
\end{eqnarray*}
We further note that from the optimality condition \eqref{eq:opt_f} we have $\wv(\alphav) = \nabla f(\vv(\alphav)) $ and rearranging terms yields:
\begin{eqnarray*}
(\Lambda)&\leq&   s g_i(u_i^{(t-1)})-s g_i(\alpha_i^{(t-1)})  - \tfrac{\mu}{2} s(1-s)(u_i^{(t-1)}-\alpha_i^{(t-1)})^2\\&&+ \underbrace{ g_i(\alpha_i^{(t-1)})+   f(\vv(\alphav^{(t-1)}))}_{(\Gamma)}+s(u_i^{(t-1)}-\alpha_i^{(t-1)})\av_i^\top \wv(\alphav^{(t-1)}) +\tfrac{1}{2 \beta}\|s(u_i^{(t-1)}-\alpha_i^{(t-1)})\av_i\|^2 .
\end{eqnarray*}
Using this inequality to bound $\bD(\vc{\alphav}{t-1})- \bD(\vc{\alphav}{t})=(\Gamma)-(\Lambda)$  yields:
\begin{eqnarray}
\bD(\vc{\alphav}{t-1})- \bD(\vc{\alphav}{t})  &\geq& -   s g_i(\vc{u_i}{t-1})+s g_i(\vc{\alpha_i}{t-1})-s\vc{u_i}{t-1}\av_i^\top \wv(\vc{\alphav}{t-1})+s \vc{\alpha_i}{t-1}\av_i^\top \wv(\vc{\alphav}{t-1})\nonumber \\
&&+ \tfrac{\mu}{2}s(1-s)(\vc{u_i}{t-1}-\vc{\alpha_i}{t-1})^2 -\tfrac{1}{2 \beta}\|s(\vc{u_i}{t-1}-\vc{\alpha_i}{t-1})\av_i\|^2.\nonumber\\
&\overset{(i)}{\geq}& s\left[g_i^*(-\av_i^\top\wv(\vc{\alphav}{t-1}))+ g_i(\vc{\alpha_i}{t-1})+\vc{\alpha_i}{t-1}\av_i^\top \wv(\vc{\alphav}{t-1}) \right.\label{eqn:D-D}\\
&&\left.+ \tfrac{\mu}{2}(1-s)(\vc{u_i}{t-1}-\vc{\alpha_i}{t-1})^2 -\tfrac{s}{2 \beta}\|(\vc{u_i}{t-1}-\vc{\alpha_i}{t-1})\av_i\|^2\right].\nonumber
\end{eqnarray}
Note that for $(i)$ we used the optimality condition \eqref{eq:opt_gstar} which translates to  $u\in \partial g^*(-\av_i^\top \wv) $ and yields $g(u)=-u\av_i^\top\wv - g^*(-\av_i^\top\wv)$. Similarly, by again exploiting the primal-dual optimality condition we have $f^*(\nabla f(\vv)) = \vv ^\top \nabla f(\vv) - f(\vv)$ and hence we can write the duality gap as:
\begin{eqnarray*}
G(\alphav)=\bP(\wv)-(- \bD({\alphav})) &=&  \sum_{i=1}^n g_i^*(-\av_i^\top \wv)+f^*(\wv)-\left[ -   \sum_{i=1}^n g_i(\alpha_i)-f(A\alphav)\right]\\
&=&  \sum_{i=1}^n\left[g_i^*(-\av_i^\top \wv)+g_i(\alpha_i)\right]+f^*(\wv)+f(A\alphav)\\
&=&  \sum_{i=1}^n\left[g_i^*(-\av_i^\top \wv)+g_i(\alpha_i)\right] +(A\alphav)^\top \wv\\
&=&  \sum_{i=1}^n\left[g_i^*(-\av_i^\top \wv)+g_i(\alpha_i) +\alpha_i\av_i^\top \wv\right]
\end{eqnarray*}
using this we can write  the expectation of \eqref{eqn:D-D} with respect to $i$ as
\begin{eqnarray*}
\Exp\left[ \bD(\vc{\alphav}{t-1})- \bD(\vc{\alphav}{t}) \right] &\geq& s \left(\frac{1}{n} G(\vc{\alphav}{t-1}) \right)
-\frac{s^2}{2} \underbrace{ \left[\frac{1}{n}\sum_{i=1}^n \Exp\left[\big(\vc{u_i}{t-1}-\vc{\alpha_i}{t-1}\big)^2\right]\left(  \frac{1}{ \beta}\|\av_i\|^2-\frac{(1-s)\mu}{ s}\right)\right]}_{ F^{(t)}}.
\end{eqnarray*}
And we have obtained that 
\begin{eqnarray*}
\Exp\left[ \bD(\vc{\alphav}{t-1})-\bD(\vc{\alphav}{t}) \right] &\geq& \frac{s}{n} G(\vc{\alphav}{t-1}) -\frac{s^2 }{2} F^{(t)}.
\end{eqnarray*}
For the expectation being over the choice of coordinate $i$ in step $t$.
\end{proof}

\begin{corollary} \label{cor:Lemma19SDCA}
We recover Lemma 19 of \cite{ShalevShwartz:2013wl} as a special case of our Lemma \ref{lemma:basic}. 
We consider their pair of primal and dual optimization objectives, \eqref{eq:primalSDCA} and \eqref{eq:dualSDCA}
where $\xv_i$ are the columns of the data matrix $X$. Assume that $\Phi^*_i$ is $\gamma$-strongly convex for $i\in[n]$, where we allow $\gamma=0$. Then, for any $t$, any $s\in[0,1]$ and $- \vc{\hat u_i}{t-1} \in \partial \Phi_i(\xv_i^\top\wv(\vc{\alphav}{t-1}))$ we have 
\begin{align}
\label{eq:lemma:D-D}
&\Exp[ -\bD(\vc{\alphav}{t}) -(- \bD(\vc{\alphav}{t-1}))] \\
& \geq \frac{s}{n} \Exp[\bP(\thetav^{t-1})-(-\bD(\alphav^{t-1}))] -\left(\frac{s}{n}\right)^2 \frac{\hat  F^{(t)}}{2\lambda}, \notag
\end{align}
where
\begin{align}
\label{eq:lemma:F}
 \vc{\hat F}{t}&:= \frac{1}{n}\sum_{i=1}^n\left( \|\xv_i\|^2-\frac{\gamma(1-s)\lambda n}{ s}\right)  \\
&\hspace{2cm} \cdot\Exp\left[\big(\vc{\hat u_i}{t-1}-\vc{\alpha_i}{t-1}\big)^2\right]\nonumber,
\end{align}
\end{corollary}

\begin{proof}
We set $g_i(\alphav):=\frac{1}{n}\Phi_i^*(-\alphav)$ and $f^*(\wv):=\frac{\lambda}{2} \|\wv\|^2$. 
From the definition of strong convexity it immediately follows that $\mu=\frac{\gamma}{ n}$ and $\beta = \lambda$.
As our algorithm works for any data matrix $A$, we choose $A:=\frac{1}{n}X$ and scale the input vectors $\xv_i$, i.e. $\av_i=\frac{\xv_i}{n}$ before we feed $\av_i$ into the algorithm. To conclude the proof we apply Lemma \ref{lemma:basic} to this setting  and observe that
\begin{eqnarray}\label{eq:l_i}
g_i^*(\wv^\top \av_i)=\tfrac{1}{n} \Phi_i(-n \wv^\top \av_i) = \tfrac{1}{n} \Phi_i( -\wv^\top \xv_i).
\end{eqnarray}
which yields  $\vc{\hat u_i}{t-1}=-n \vc{ u_i}{t-1}$.
Further note that  the conjugate of $f^*$ is given by $f(\vv)= \frac{\lambda}{2} \left\|\frac{\vv}{\lambda}\right\|^2$ and this leads to the dual-to-primal mapping  
\[\wv = \nabla f (\vv(\alphav))= \frac{1}{\lambda}{\vv(\alphav)}=\frac{1}{\lambda}\sum_{i=1}^n {\av_i\alpha_i}=\frac{1}{\lambda n}\sum_{i=1}^n {\xv_i\alpha_i}.\]
\end{proof}

\section{Some Useful Pairs of Conjugate Functions}

\paragraph{Elastic Net.}

\begin{lemma}[Conjugate of the Elastic Net Regularizer, see e.g. \cite{Smith:2015ua}]
\label{lem:elasticnetconjugate}
For $\eta \in (0,1]$, the Elastic Net function $g_i(\alpha) := \frac{\eta}{2} \alpha^2 + (1-\eta) |\alpha|$ has the convex conjugate 
\[
    g_i^*(x) := \textstyle\frac1{2\eta} \big(\big[|x|-(1-\eta)\big]_+\big)^2 ,
\]
where $[.]_+$ is the positive part operator, $[s]_+ = s$ for $s>0$, and zero otherwise.
Furthermore, this $g_i^*$ is $1/\eta$-smooth.
\end{lemma}
\begin{proof}%
We start by applying the definition of convex conjugate, that is:
$$\textstyle 
\ell(x) = \max_{\alpha \in \R} \left[ x\alpha - \eta \frac{\alpha^2}{2} - (1-\eta)|\alpha| \right] \, .\vspace{-1mm}
$$

We now distinguish two cases for the optimal: $\alpha^\star \geq 0$, $\alpha^\star < 0$.
For the first case we get that
$$\textstyle 
\ell(x) = \max_{\alpha \in \R} \left[ x\alpha - \eta \frac{\alpha^2}{2} - (1-\eta)\alpha \right] \, .
$$
Setting the derivative to $0$ we get $\alpha^\star = \frac{x-(1-\eta)}{\eta}$. To satisfy $\alpha^\star \geq 0$, we must have $x \geq 1-\eta$.
Replacing with $\alpha^\star$ we thus get:
$$\textstyle 
\ell(x) = \alpha^\star (x - \frac{1}{2}\eta \alpha^\star -(1-\eta)) = \alpha^\star \left( x - \frac{1}{2} (x-(1-\eta)) - (1-\eta) \right) =
$$
$$\textstyle 
\frac{1}{2} \alpha^\star \left( x - (1-\eta) \right) = \frac{1}{2\eta} \left( x - (1-\eta) \right)^2 \, .
$$
Similarly we can show that for $x \leq -(1-\eta)$
$$\textstyle 
\ell(x) = \frac{1}{2\eta} \left( x + (1-\eta) \right)^2 \, .
$$
Finally, by the fact that $\ell(.)$ is convex, always positive, and $\ell(-(1-\eta)) = \ell(1-\eta) = 0$,
it follows that $\ell(x) = 0$ for every $x \in [-(1-\eta),1-\eta]$.

For the smoothness properties, we consider the derivative of this function $\ell(x)$ and see that $\ell(x)$ is smooth, i.e. has Lipschitz continuous gradient with constant $1/\eta$, assuming $\eta>0$. Alternatively, use Lemma~\ref{lem:dualSmooth} for the given strongly convex function $g_i$.
\end{proof}

\paragraph{Group Lasso.} 
The group Lasso regularizer is a norm on $\R^n$, and is defined as 
\[
g(\alphav)=%
\lambda \sum_{k} \|\alphav_{\calG_k}\|_2 ,\vspace{-1mm}
\]
for a fixed partition of the indices into disjoint groups, $\{1..n\} = \biguplus_{k} \calG_k$. 
Here $\alphav_{\calG_k} \in \R^{|g|}$ denotes the part of the vector $\alphav$ with indices in the group $g \subseteq [n]$.
Its dual norm is $\max_{g\in G} \|\alphav_{(g)}\|_2$. Therefore, by Lemma~\ref{lem:NormConjugates}, we obtain the conjugate
\[
g^*(\yv) = \id_{\{\yv \ |\ \max_{g\in G} \|\alphav_{(g)}\|_2 \,\le\, \lambda \}}(\yv)
\] %
See, e.g. \citep[Example 3.26]{Boyd:2004uz}.

\paragraph{Logistic Loss.}

\begin{lemma}[Conjugate and Smoothness of the Logistic Loss]
\label{lem:logisticlossconj}
The logistic classifier loss function $f$ is given as\vspace{-1mm}
\begin{equation*}
f(A\alphav) := \sum_{j=1}^d \log{(1 + \exp{(-b_j \yv_j^\top \alphav)})} \, , \vspace{-2mm}
\end{equation*}
Its conjugate $f^*$ is given as:\vspace{-2mm}
\begin{equation}
\label{eq:logisticlossconj}
f^*(\wv) := \sum_{j=1}^d \big((1+w_j b_j) \log{(1+w_j b_j)} - w_j b_j\log{(-w_j b_j)\big)} \, ,
\end{equation}
with the box constraint $-w_jb_j \in [0,1]$.

Furthermore, $f^*(\wv)$ is $1$-strongly convex over its domain if the labels satisfy $b_j\in[-1,1]$.\vspace{-1mm}
\end{lemma}
\begin{proof}
See e.g. \cite{Smith:2015ua} or \cite{ShalevShwartz:2013wl}.
\end{proof}

\newpage
\section{More Numerical Experiments}

\includegraphics[scale=0.2]{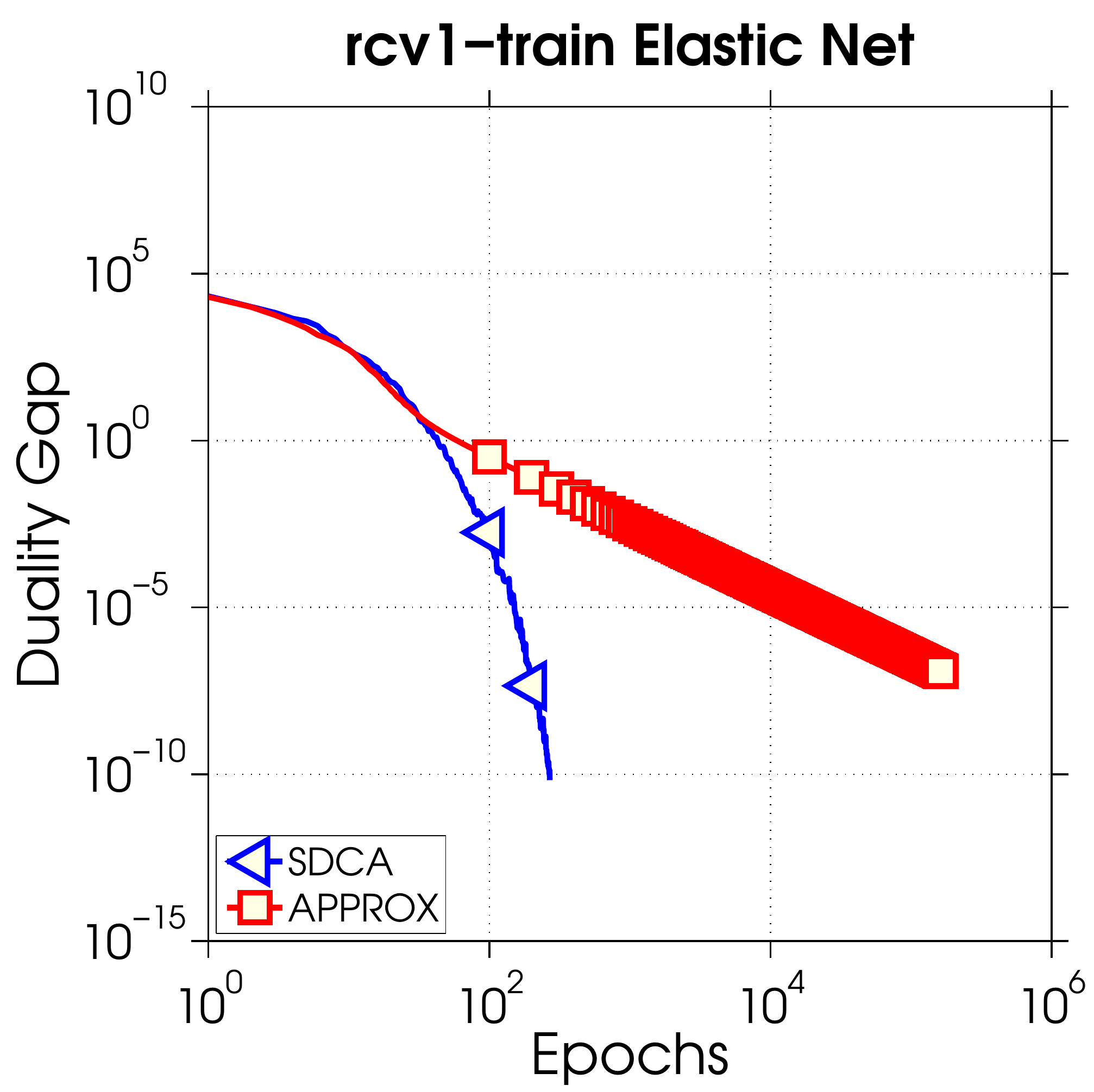}
\includegraphics[scale=0.2]{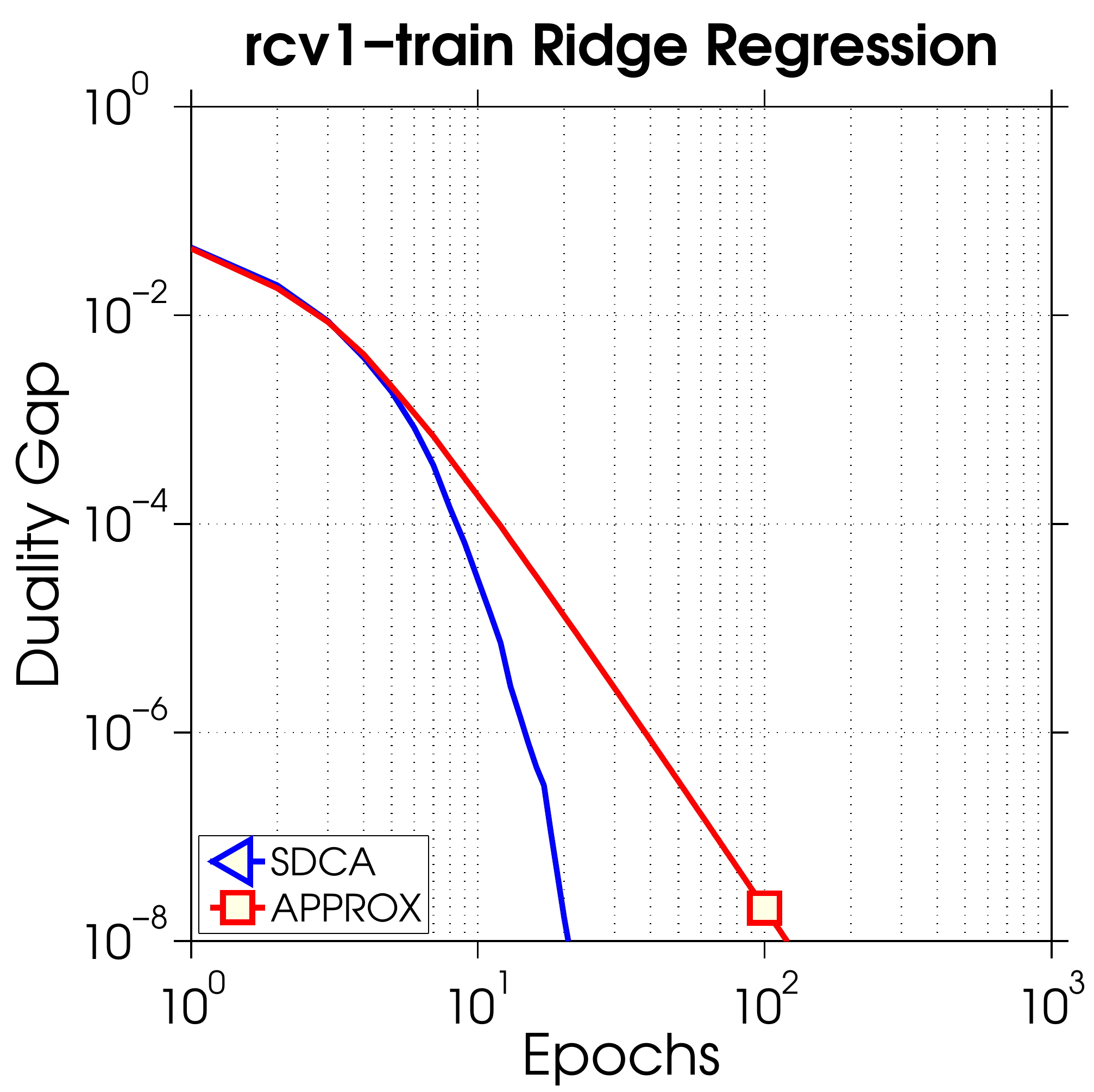}
\includegraphics[scale=0.2]{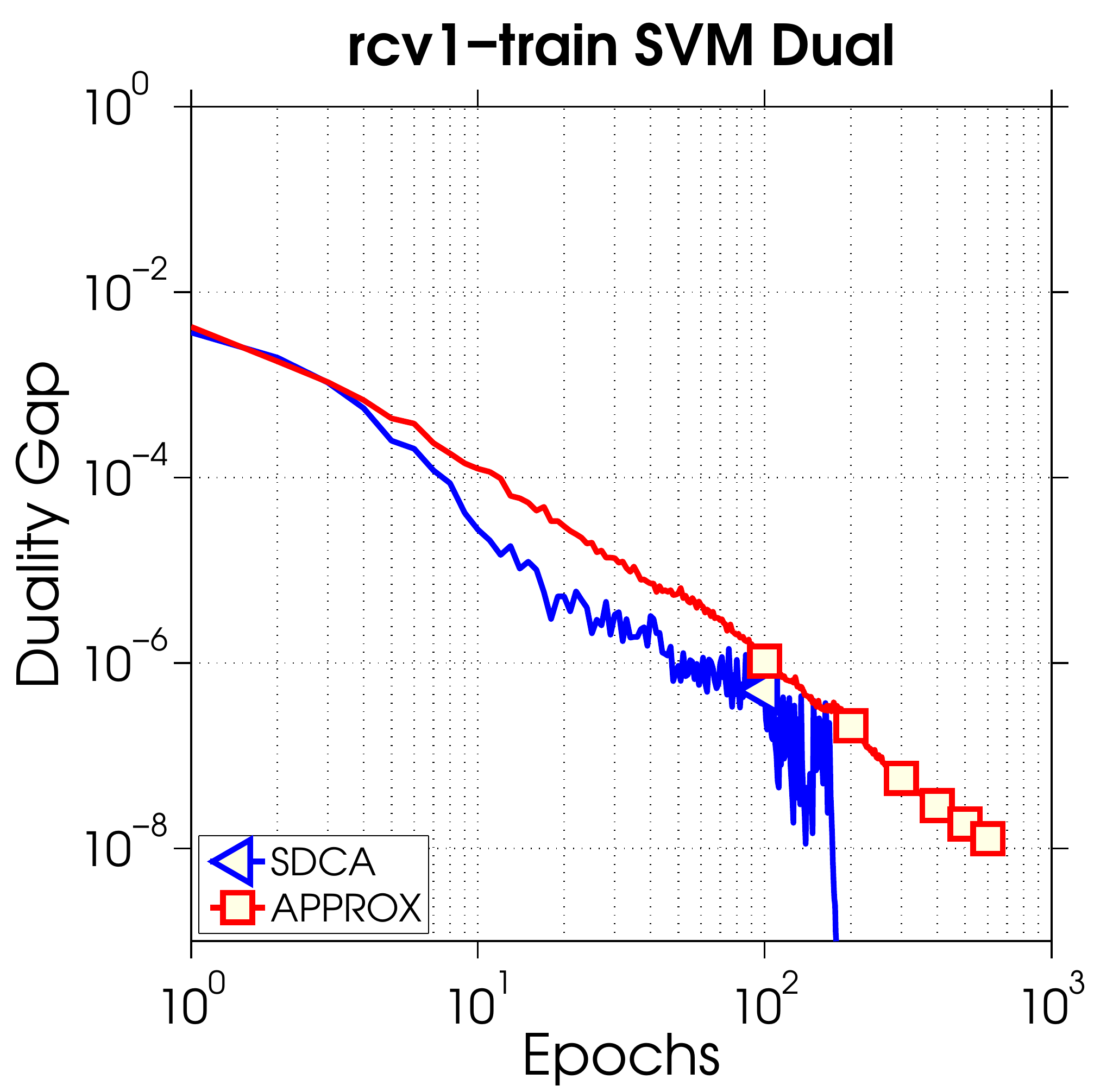}
\includegraphics[scale=0.2]{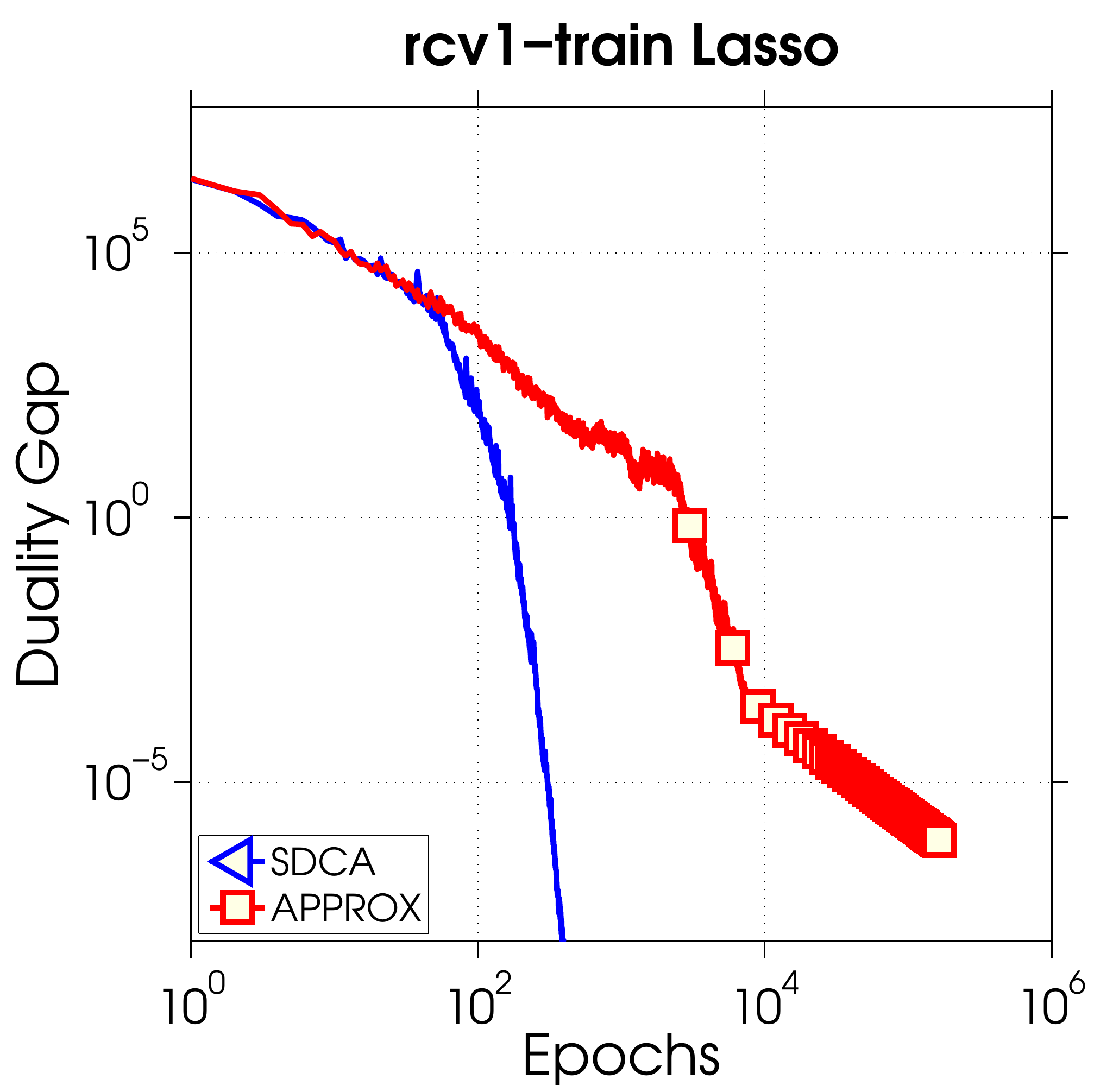}

\includegraphics[scale=0.2]{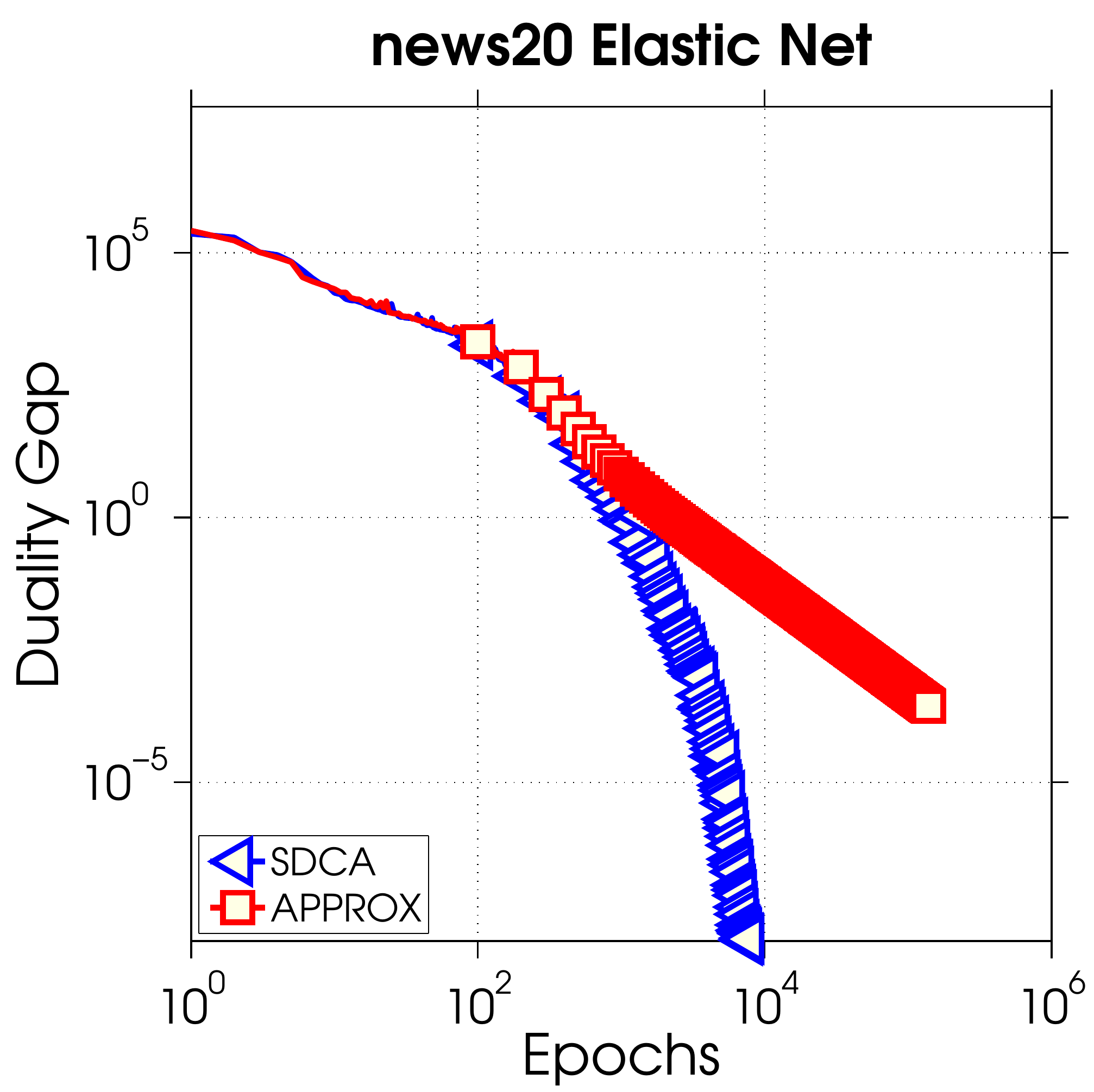}
\includegraphics[scale=0.2]{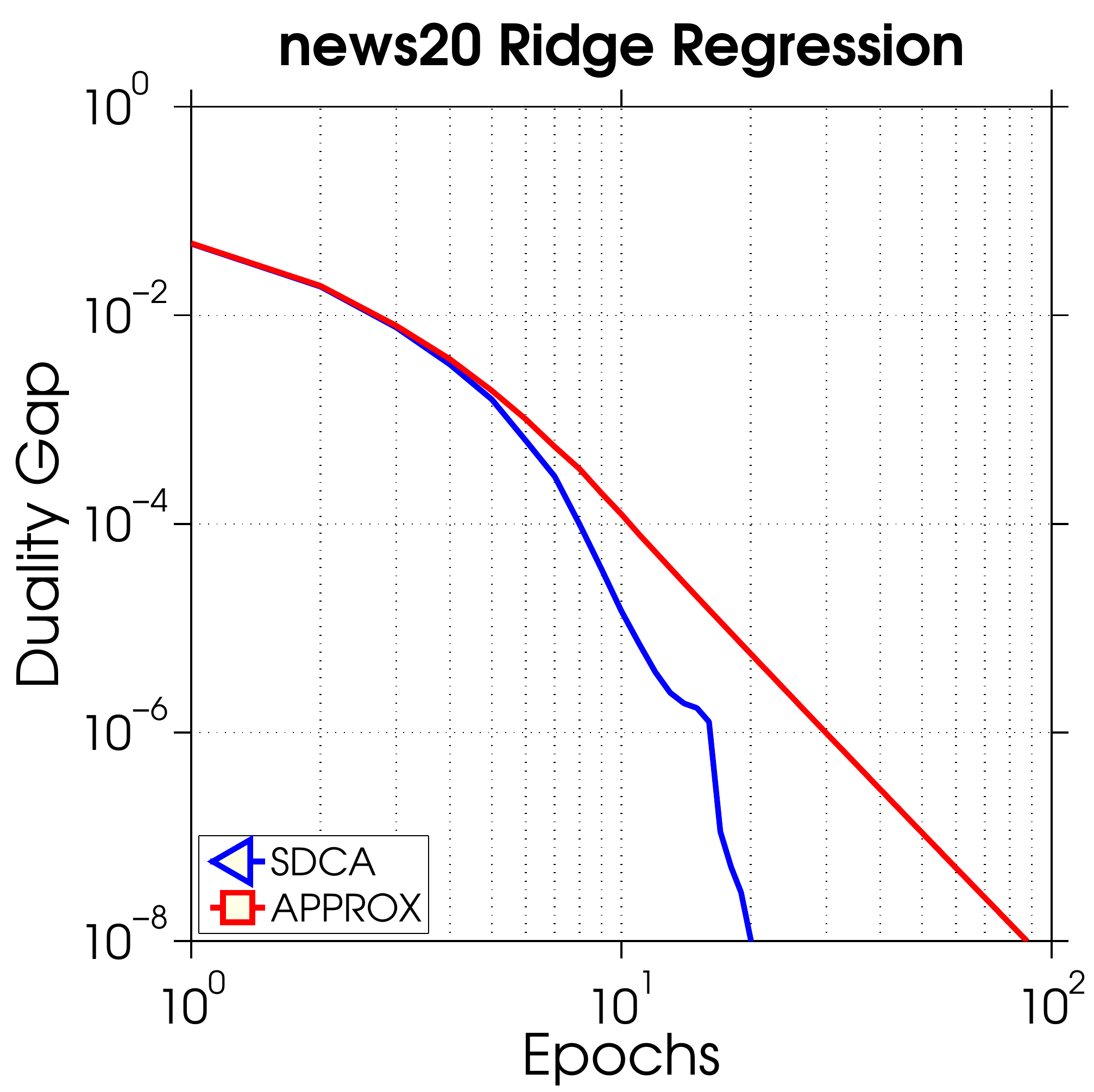}
\includegraphics[scale=0.2]{exp_2_3-eps-converted-to.pdf}
\includegraphics[scale=0.2]{exp_2_4-eps-converted-to.pdf}

\includegraphics[scale=0.2]{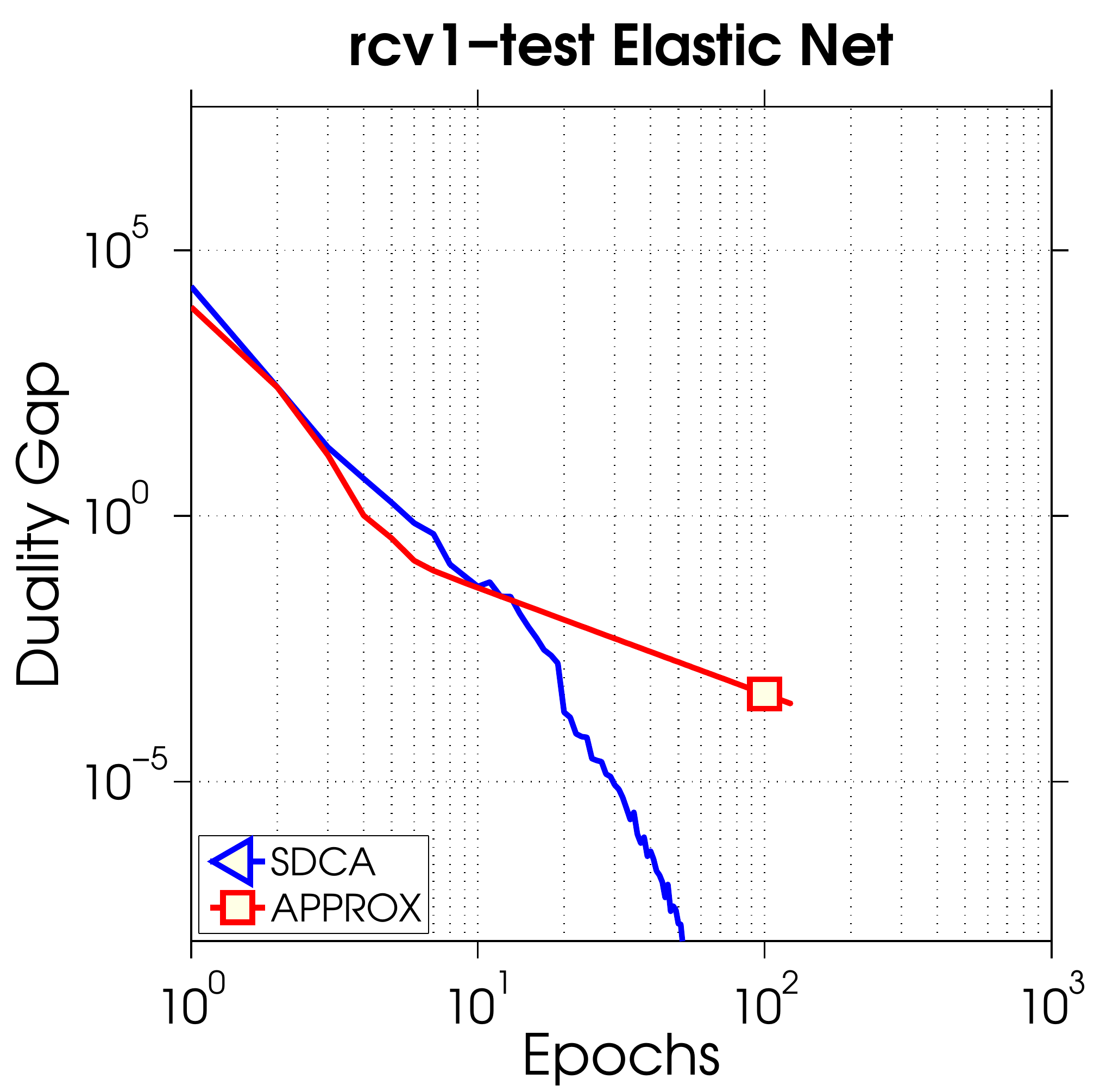}
\includegraphics[scale=0.2]{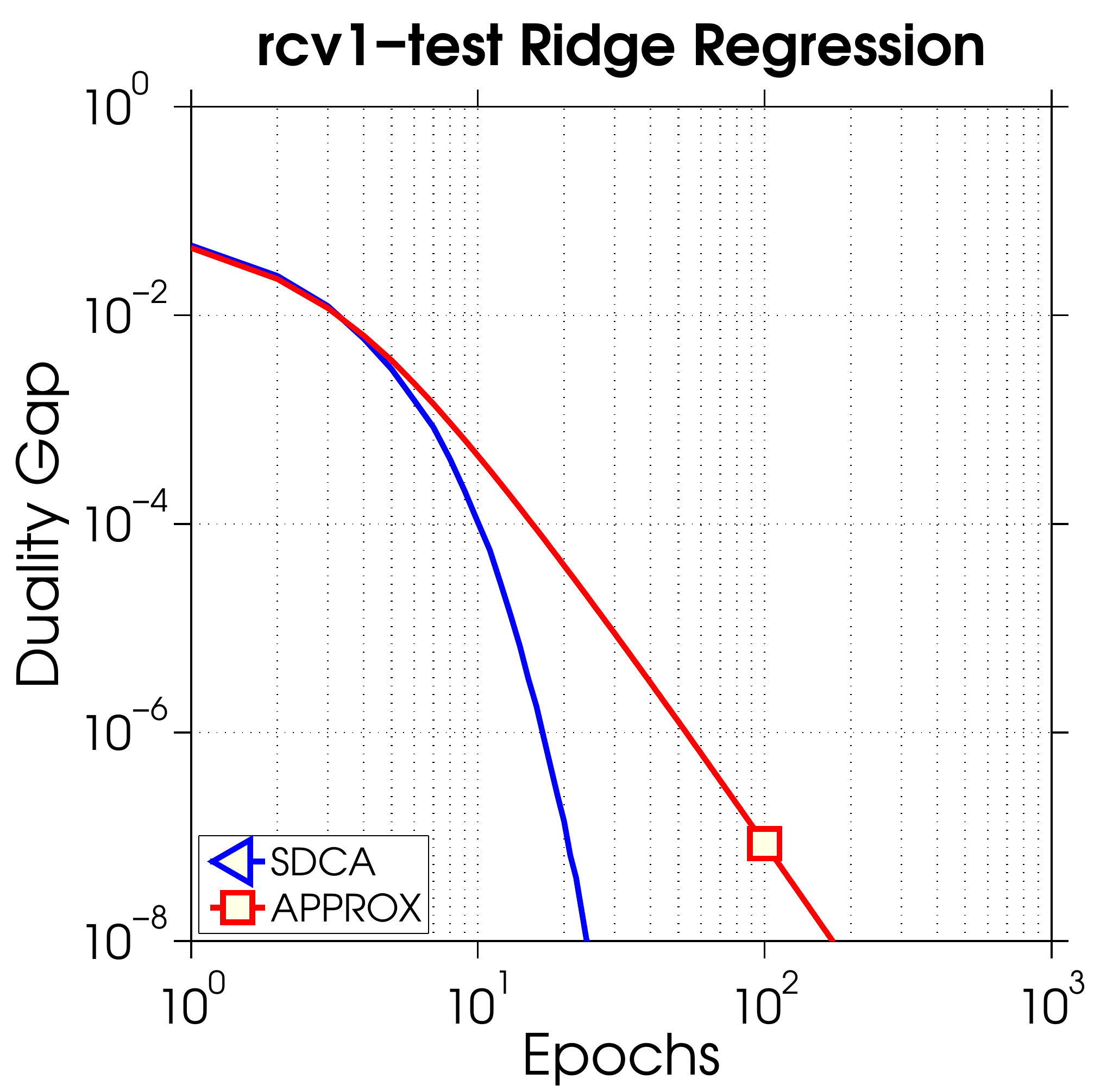}
\includegraphics[scale=0.2]{exp_6_3-eps-converted-to.pdf}
\includegraphics[scale=0.2]{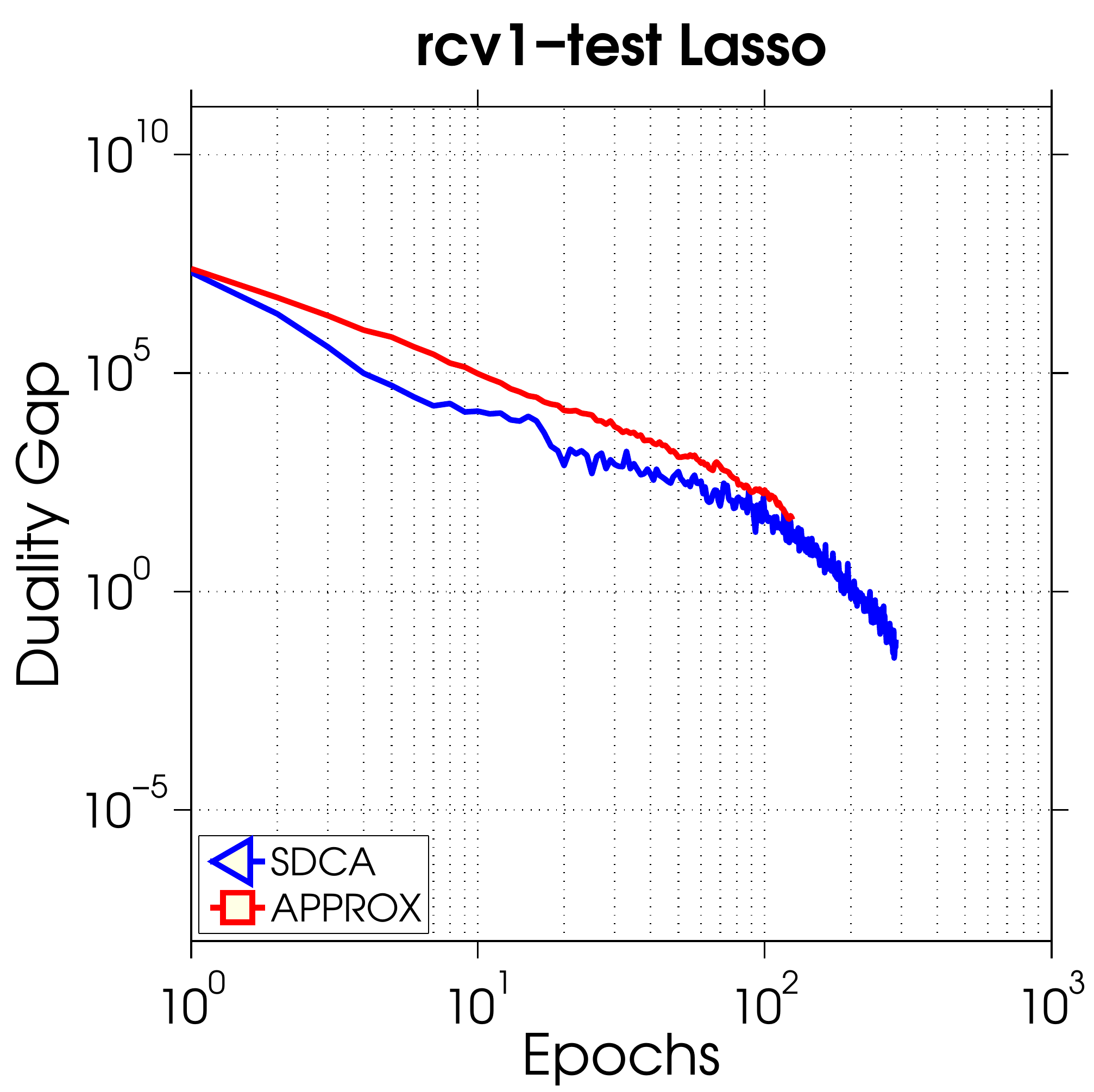}

\includegraphics[scale=0.2]{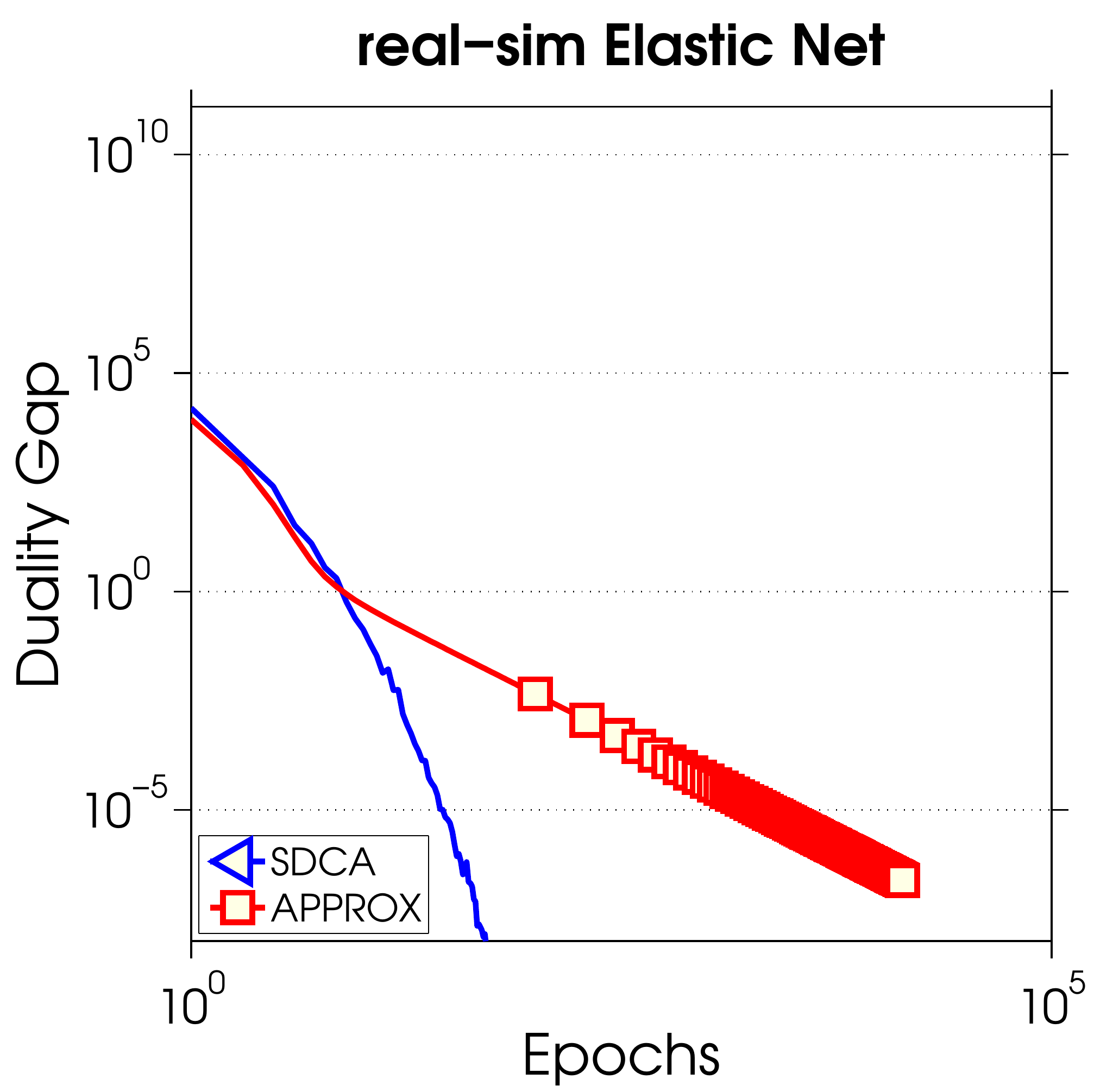}
\includegraphics[scale=0.2]{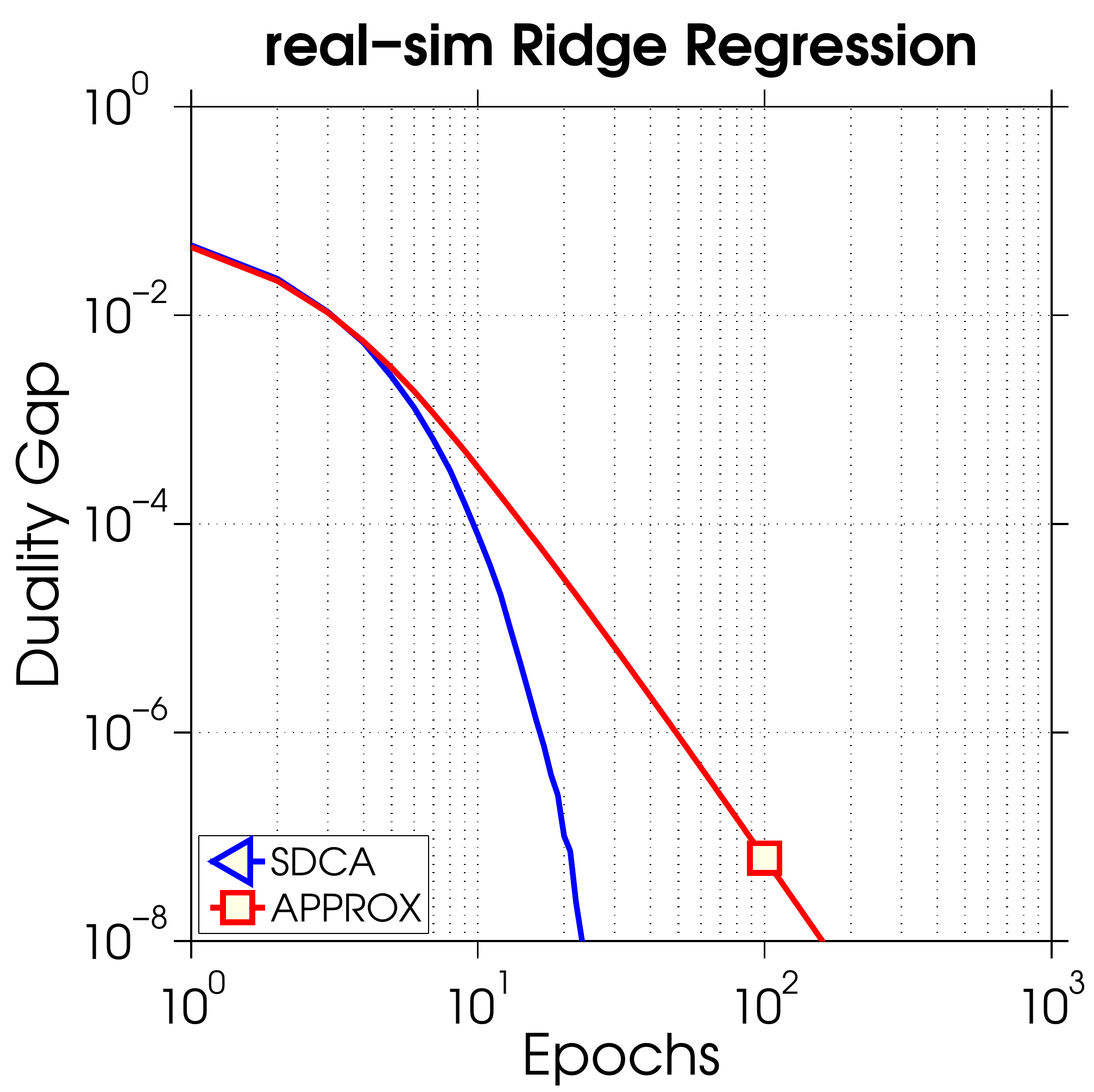}
\includegraphics[scale=0.2]{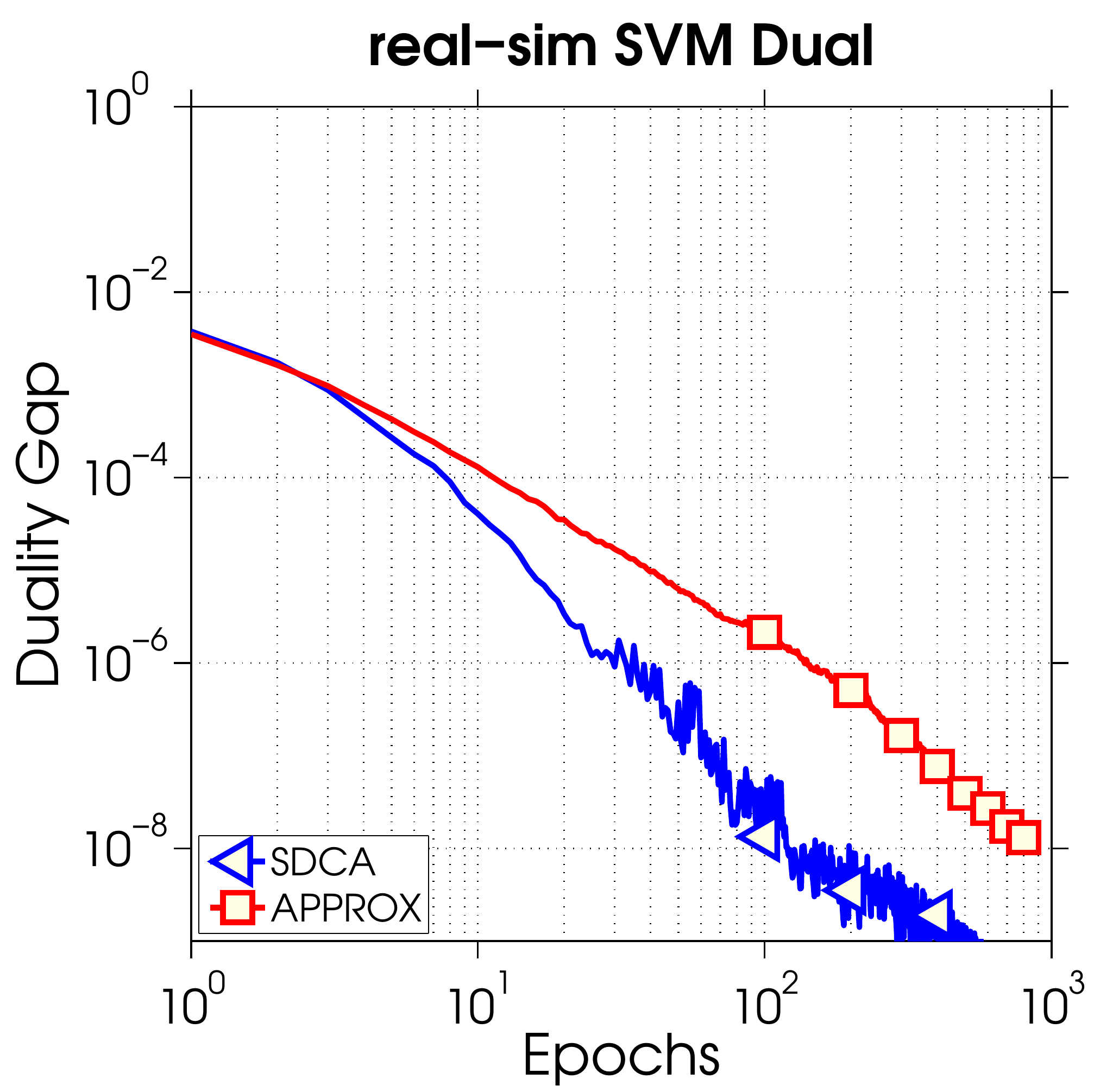}
\includegraphics[scale=0.2]{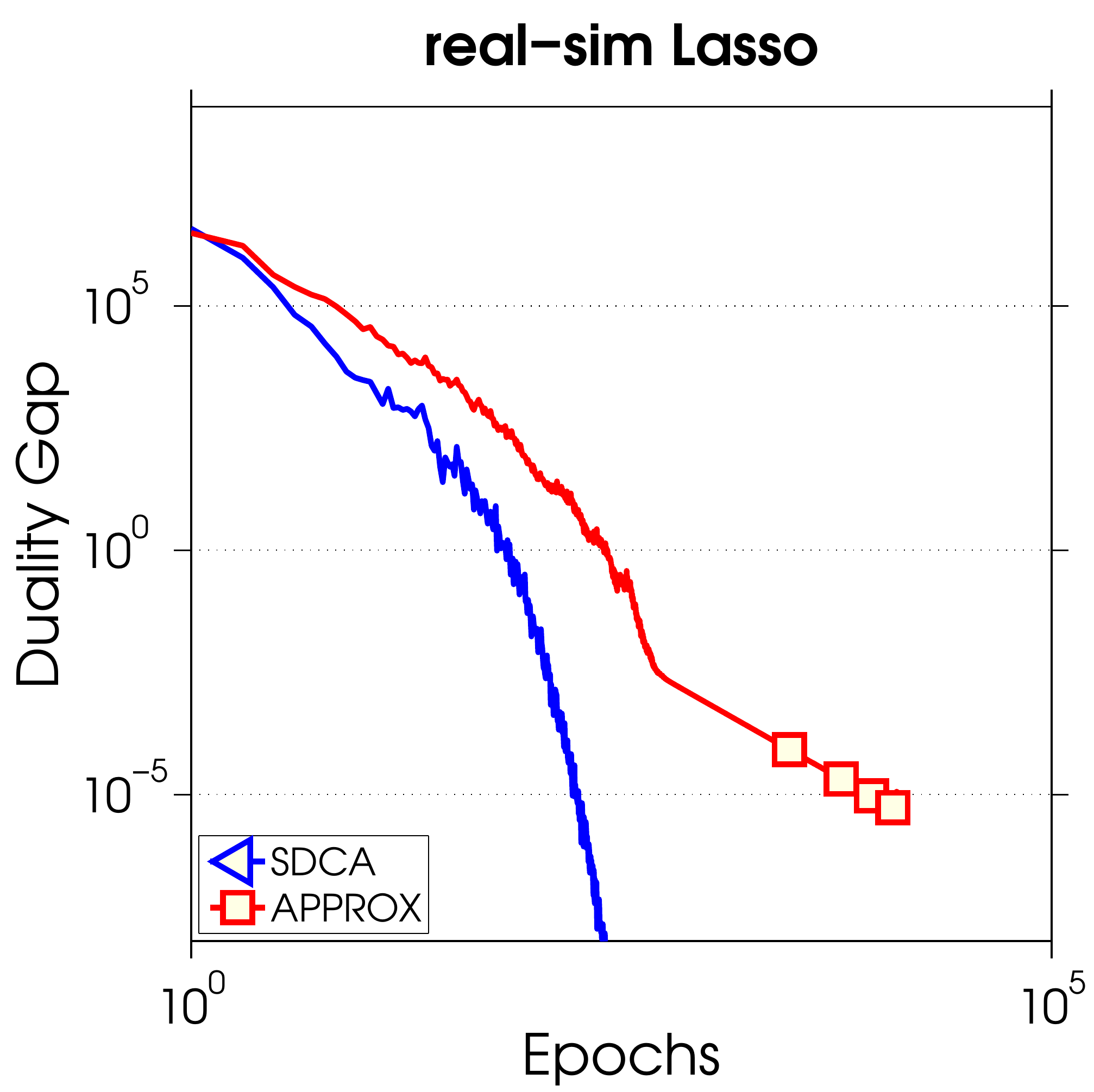}

\end{document}